 \documentclass{jmlr} % 
\usepackage{amsmath,amssymb,natbib,graphicx,url,algorithm2e}
\usepackage{times}
\usepackage{xcolor}
\usepackage{bm}
\jmlrproceedings{}{}
\jmlrpages{}

\DeclareMathOperator*{\E}{\mathbb{E}}
\DeclareMathOperator{\kl}{KL}
\DeclareMathOperator{\bernoullikl}{B}
\DeclareMathOperator*{\argmin}{arg\,min}

\DeclareMathOperator{\Tr}{Tr}			% 

\DeclareBoldMathCommand{\u}{u}      	% 
\DeclareBoldMathCommand{\w}{w}      	% 
\DeclareBoldMathCommand{\x}{x}      	% 
\DeclareBoldMathCommand{\g}{g}      	% 
\DeclareBoldMathCommand{\z}{z}      	% 
\DeclareBoldMathCommand{\0}{0}			% 
\DeclareBoldMathCommand{\np}{\theta}	% 
\DeclareBoldMathCommand{\Mu}{\mu}		% 
\DeclareBoldMathCommand{\e}{e}			% 
\DeclareBoldMathCommand{\1}{1}			% 
\DeclareBoldMathCommand{\I}{I}			% 
\DeclareBoldMathCommand{\M}{M}			% 

\renewcommand{\top}{\intercal}
\newcommand{\reals}{\mathbb{R}}
\newcommand{\half}{\frac{1}{2}}
\newcommand{\thalf}{\tfrac{1}{2}}       % 
\newcommand{\domainP}{\mathcal{P}}
\newcommand{\domainw}{\mathcal{W}}
\newcommand{\normal}{\mathcal{N}}       % 
\newcommand{\simplex}{\triangle^{{\kern-0.1em}d}}
\newcommand{\Expfam}{\mathcal{E}}		% 
\newcommand{\inner}[2]{\langle #1, #2 \rangle}  % 
\newcommand{\sloss}{\ell}                       % 
\newcommand{\altsloss}{\tilde{\ell}}            % 
\newcommand{\sregret}{S}                        % 
\newcommand{\altsregret}{\tilde{S}}             % 
\newcommand{\altbeta}{\tilde{\beta}}            % 
\newcommand{\altw}{\tilde{w}}                   % 
\newcommand{\altp}{\tilde{p}}                   % 
\newcommand{\etaopt}{\hat{\eta}}                % 
\newcommand{\der}{\mathrm{d}}                   % 
\newcommand{\postpi}{\hat{\pi}}                 % 
\newcommand*\dif{\mathop{}\!\textnormal{d}}
\newcommand{\regret}{\mathcal{R}}
\newcommand{\V}{\mathcal{V}}                    % 

\DeclareRobustCommand{\VAN}[3]{#2} % 

\title[The Many Faces of Exponential Weights]{The Many Faces of
Exponential Weights in Online Learning}
\author{\Name{Dirk {van der Hoeven}} \Email{dirkvderhoeven@gmail.com}\\
\and
\Name{Tim {van Erven}} \Email{tim@timvanerven.nl}\\
\addr Statistics Department, Leiden University, the Netherlands
\AND
\Name{Wojciech Kot{\l}owski} \Email{wkotlowski@cs.put.poznan.pl}\\
\addr Institute of Computing Science, Pozna{\'n} University of Technology, Poland
}

\begin{document}
\maketitle

\begin{abstract}
A standard introduction to online learning might place Online Gradient
Descent at its center and then proceed to develop generalizations and
extensions like Online Mirror Descent and second-order methods. Here we
explore the alternative approach of putting Exponential Weights (EW)
first. We show that many standard methods and their regret bounds then
follow as a special case by plugging in suitable surrogate losses and
playing the EW posterior mean. For instance, we easily recover Online
Gradient Descent by using EW with a Gaussian prior on linearized losses,
and, more generally, all instances of Online Mirror Descent based on
regular Bregman divergences also correspond to EW with a prior that
depends on the mirror map. Furthermore, appropriate quadratic surrogate
losses naturally give rise to Online Gradient Descent for strongly
convex losses and to Online Newton Step. We further interpret several
recent adaptive methods (iProd, Squint, and a variation of Coin Betting
for experts) as a series of closely related reductions to exp-concave
surrogate losses that are then handled by Exponential Weights. Finally,
a benefit of our EW interpretation is that it opens up the possibility
of sampling from the EW posterior distribution instead of playing the
mean. As already observed by \citeauthor{bubeck2014}, this recovers the
best-known rate in Online Bandit Linear Optimization.
\end{abstract}

\section{Introduction}

\emph{Exponential Weights} (EW) \citep{Vovk1990,LittlestoneWarmuth1994} is a method for keeping track of uncertainty about the
best action in sequential prediction tasks. It is most commonly
considered for a finite number of actions in the prediction with expert
advice setting, where each of the actions corresponds to following the
advice of one of a finite number of experts, and in this context it is
asymptotically minimax optimal
\citep[Section~2.2]{CesaBianchiLugosi2006}. However, in the present work
we mostly consider EW on continuous action spaces in the more general
setting of Online Convex Optimization \citep{hazan2016introduction},
where we show that surprisingly many standard methods turn out to be
special cases of EW.

EW keeps track of a probability distribution over actions
that is updated in each round of the prediction task by multiplying the
probability of each action by a factor that is exponentially decreasing
in the action's error or \emph{loss} in that round, and renormalizing.
This type of update is quite flexible: by assigning appropriate
surrogate losses to the actions, it covers any kind of multiplicative
probability updates, including, for instance, those of the Prod
algorithm \citep{CesaBianchiMansourStoltz2007}. For best performance,
losses often need to be scaled by a positive parameter called the
learning rate, and the algorithm may also be biased towards particular
actions by the choice of its initial distribution, which is called the
prior. For continuous sets of actions, efficient implementations of EW
are often restricted to conjugate priors for which the EW distribution
can be analytically computed, but sampling approximations based on
random walks can also provide appealing trade-offs between computational
complexity and prediction accuracy, even for a single random walk step
per round
\citep{NarayananRakhlin2017,kalai2002efficient}.

The usual presentation of Online Convex Optimization would introduce EW
as a special case of Mirror Descent (MD) or Follow-the-Regularized-Leader
(FTRL) with the Kullback-Leibler divergence as the regularizer. However,
here we turn this view on its head
and show that all instances of MD based on regular Bregman
divergences \citep{banerjee2005} in fact
correspond to EW on a continuous set of actions (Section~\ref{sec:MD}). In particular, Gradient
Descent (GD) comes from using a Gaussian prior on linearized losses
(Section~\ref{sec:GD}), which is striking because GD has been contrasted
with the Exponentiated Gradient Plus-Minus algorithm
\citep{KivinenWarmuth} that is readily seen to be an instance of EW
(Section~\ref{sec:EG}). In addition, the unnormalized relative entropy
regularizer \citep{HelmboldWarmuth2009}, which is normally considered a
generalization of EW, turns out to be a special case of EW as well for
a multivariate Poisson prior (Section~\ref{sec:MD}). Furthermore, in
Section~\ref{sec:quadlosses} we show that running
EW on suitable quadratic approximations of the losses recovers Gradient Descent for strongly convex losses
\citep{HazanAgarwalKale2007} and, as already observed by
\citet{VanErvenKoolen2016}, Online Newton Step
\citep{HazanAgarwalKale2007}. The Vovk-Azoury-Warmuth forecaster would
also be an example of running EW on quadratic losses, but we refer to
\citep{Vovk2001} for its analysis, which requires a generalized proof
technique (see also the discussion by
\citet{OrabonaCrammerCesaBianchi}). We do consider the recent adaptive
iProd, Squint and Coin Betting methods of
\citet{KoolenVanErven2015,OrabonaPal2016}, which learn the optimal
learning rate for prediction with expert advice, and show that these may
also be viewed as running EW after a reduction of the original
prediction task to various closely related surrogate tasks in which the
learning rate is just one of the parameters that does not need to be
treated specially (Section~\ref{sec:adaptivity}). Finally, in the
context of Bandit Linear Optimization, the SCRiBLe method
\citep{abernethy2008} may be viewed as an approximation to EW, and an
application of EW outlined by \citet{bubeck2014} achieves the best-known
rate (we provide the technical details they~omit~in~Section~\ref{sec:bandits}).

\paragraph{Related Work}

The diverse applications of EW on a finite number of actions range, for
instance, from boosting \citep{FreundSchapire1997} to differential
privacy \citep{DworkBook2014} to multi-armed bandits \citep{aueretal2002}, and many algorithms in
computer science can be viewed as special cases of EW \citep{Arora}.
EW has also been considered for continuous sets of actions, often in the context of universal coding in information theory, where
the goal is to sequentially compress a sequence of symbols. In this
case, actions parametrize a set of probability distributions and the
loss of an action is the logarithmic loss for the corresponding
probability distribution on the symbol that is being compressed
\citep[Chapter~9]{CesaBianchiLugosi2006}. EW (with learning rate~$1$) then simplifies to Bayesian probability updating. The choice of
prior has received much attention in this literature, with Jeffreys'
prior being shown to be asymptotically minimax optimal for exponential
families with parameters restricted to suitable bounded sets 
\citep[Chapter~8]{Grunwald2007}. Without parameter restrictions,
Jeffreys' prior is still minimax optimal up to constants for the
Bernoulli and multinomial models
\citep{krichevsky1981performance,XieBarron}. Several applications to
other losses are also closely related to the log loss: Online
Ridge Regression corresponds to EW on the squared loss, which matches
the log loss for Gaussian distributions; and Cover's method for
portfolio selection \citep{Cover1991}, which is EW on
Cover's loss, may be interpreted as
learning a mixture model under the log loss
\citep{orseau2017soft}. In general, continuous EW is not restricted to
the log loss, however, and has been considered e.g.\ for general convex
losses \citep{DickGyorgyCzepesvari2014} or as a computationally
inefficient gold standard for exp-concave losses
\citep{HazanAgarwalKale2007}.

\section{Exponential Weights}
\label{sec:EW}

In Online Convex Optimization (OCO)
\citep{ShaiBook,hazan2016introduction} a learner repeatedly chooses
actions $\w_t$ from a convex set $\domainw \subseteq \reals^d$ during
rounds $t=1,\ldots,T$, and suffers losses~$f_t(\w_t)$, where $f_t:
\domainw \to \reals$ is a convex function. The learner's 
goal is to achieve small \emph{regret}
  $\regret_T(\u) = \sum_{t=1}^T f_t(\w_t) - \sum_{t=1}^T f_t(\u)$
with respect to any comparator action $\u \in \domainw$,
which measures the difference between the cumulative loss of the learner and
the cumulative loss it could have achieved by playing the oracle
action~$\u$ from the start.
We will assume the domain of the losses $f_t$ is extended from $\domainw$ to
$\reals^d$ with convexity of $f_t$ being preserved. This comes without
loss of generality as one can always set $f_t(\w)=\infty$ outside
$\domainw$, but we will use more natural and straightforward extensions
throughout the paper (e.g.\ when the $f_t$ are linear or quadratic
functions).

The central topic of this work is the Exponential Weights (EW)
algorithm, which keeps track of uncertainty over actions expressed by a
distribution $P_t$ and comes in the two flavors shown in
Figure~\ref{alg:EW} (our naming follows \citet{zinkevich2003}), where
we let $\kl(P\|Q) = \E_{P}\big[ \ln \frac{\der P}{\der Q}
\big]$ denote the Kullback-Leibler (KL) divergence between distributions
$P$ and $Q$. 
\begin{figure}
{% 
\centering \small
\begin{tabular}{| c | c | r |}
\hline
\multicolumn{2}{|l|}{\textbf{Input:} a convex set of distributions
$\domainP$ over $\w$, a prior $P_1 \in \domainP$ and learning rates $\eta_1 \geq \eta_2 \geq \cdots \geq \eta_T > 0$} \\
\hline
Lazy Exponential Weights & Greedy Exponential Weights \\
\hline
\parbox{7.4cm}{\vspace{-0.35cm} \begin{equation*}
\begin{split}
    \tilde{P}_{t+1} & = \argmin_P ~ \E_P\big[\textstyle\sum_{s=1}^t f_s(\w)\big] + \frac{1}{\eta_{t}}\kl(P\|P_1) \\
    P_{t+1} & = \argmin_{P \in \domainP} ~ \kl(P\|\tilde{P}_{t+1})
\end{split}
\end{equation*} \vspace{-0.35cm} } & \parbox{6cm}{\vspace{-0.35cm} \begin{equation*}
\begin{split}
    \tilde{P}_{t+1} &= \argmin_P ~ \E_P[f_t(\w)] +
    \tfrac{1}{\eta_{t}}\kl(P\|P_t)\\
    P_{t+1} & = \argmin_{P \in \domainP} ~ \kl(P\|\tilde{P}_{t+1}) % 
\end{split}
\end{equation*} \vspace{-0.35cm}} \\
\hline
\end{tabular}
}
\caption{The
lazy and greedy versions of Exponential Weights}
\label{alg:EW}
\end{figure}
The algorithm gets its name from the distributions $\tilde{P}_t$, whose
densities have the following exponential forms:
\begin{align}
\dif \tilde{P}_{t+1}(\w) &= \frac{e^{-\eta_t \sum_{s=1}^t f_s(\w)}\dif P_1(\w)}{\int e^{-\eta_t \sum_{s=1}^t f_s(\w)}\dif P_1(\w)}
&& \text{(lazy EW)}
\label{eq:lazy_EW}\\
\dif \tilde{P}_{t+1}(\w) &=
\frac{e^{-\eta_t f_t(\w)} \dif P_t(\w)}{\int e^{-\eta_t f_t(\w)} \dif
P_t(\w)} && \text{(greedy EW).}
\label{eq:greedy_EW}
\end{align}
In the case that $\domainP$ contains all possible distributions over
$\reals^d$ (for which the projection step becomes void) and the
\emph{learning rates} $\eta_t$ are constant $\eta_1 = \cdots = \eta_T =
\eta$, both versions of EW are equivalent. In general they differ, and
enjoy the following regret bounds with respect to a potentially
randomized comparator drawn from a comparator distribution $Q$, which
follow from a standard MD analysis \citep{hazan2016introduction}
and a reformulation of the standard FTRL analysis that works for
distributions $P_t$ on continuous spaces, which cannot be expressed as
the finite-dimensional vectors that are usually assumed (the proof
details are in Appendix~\ref{sec:EWRegretProof}):
\begin{lemma}[EW Regret]\label{lem:EWRegret}
  Suppose that $\eta_1 \geq  \eta_{2} \geq \ldots
  \geq \eta_T > 0$, and that the minima that define
  $\tilde{P}_t$ and $P_t$ are uniquely achieved. Let $Q \in \domainP$ be
  any comparator distribution such that $\kl(Q\|\tilde{P}_t) < \infty$
  for all $t$, let $\{\w_t \in \domainw\}_{t=1}^T$ be the actions of
  any learner, and define $\eta_0 \stackrel{\mathrm{def}}{=} \eta_1$. Then  EW satisfies
\begin{align}
    \E_{\u \sim Q}[\regret(\u)]
      &\leq \frac{1}{\eta_T} \kl(Q \| P_1)
      + \sum_{t=1}^T \Big\{ \underbrace{f_t(\w_t) + \frac{1}{\eta_{t-1}}
          \ln \E_{P_t(\w)}\Big[e^{-\eta_{t-1}
        f_t(\w)}\Big]}_{\textnormal{``mixability gap''}}\Big\}
      && \textnormal{(lazy EW)}
    \label{eqn:EWRegret}\\
    \E_{\u \sim Q}[\regret(\u)]
      &\leq \frac{1}{\eta_1} \kl(Q \| P_1)
        + \left(\frac{1}{\eta_T}-\frac{1}{\eta_1}\right)\max_{t=2,\ldots,T} \kl(Q\|P_t)
        \nonumber \\
       &\quad+ \sum_{t=1}^T \Big\{ \underbrace{f_t(\w_t) + \frac{1}{\eta_t}
        \ln \E_{P_t(\w)}\Big[e^{-\eta_t
        f_t(\w)}\Big]}_{\textnormal{``mixability gap''}}\Big\}
    && \hspace{-1em}\textnormal{(greedy EW).}\label{eqn:EWRegret_greedy}
  \end{align}
\end{lemma}
While the predictions $\w_t$ in Lemma~\ref{lem:EWRegret} are arbitrary
actions from $\domainw$, one always chooses $\w_t$ to be some function
of $P_t$. A general mapping from $P_t$ to $\w_t$ is called a
\emph{substitution function} \citep{Vovk2001} and is usually designed to
give the best bound on the mixability gap in trial $t$. 
Throughout the paper, we will use the mean $\w_t = \E_{P_t}[\w]$ as our
substitution function, which is a typical choice, although alternatives
may be better
in specific cases \citep{Vovk2001}. To
ensure that $\w_t \in \domainw$, we will also generally assume that
$\domainP = \{P : \E_P[\w] \in \domainw\}$, which is convex.

Bounding the mixability gap is a crucial part of the regret analysis of
EW \citep{Vovk2001,DeRooijVanErvenGrunwaldKoolen2014}. In the special
case that the losses are \emph{$\alpha$-exp-concave} for $\alpha > 0$
(i.e.\ if $e^{-\alpha f(\w)}$ is concave), the mixability gap for
$\eta_t \leq \alpha$ is at most $0$. This happens in the following
example.

\begin{example}[The Krichevsky-Trofimov Estimator]\label{ex:KT}
  Let $\domainw = [0,1]$ and let the loss function be the log loss: $f_t(w) = 
  -x_t\ln(w) - (1-x_t)\ln(1-w)$, where $x_t \in \{0, 1\}$. A standard
  algorithm in this case is the 
  Krichevsky-Trofimov forecaster $w_t = (\sum_{s=1}^{t-1}x_s +
  \half)/t$
  \citep[Chapter~9]{CesaBianchiLugosi2006}, which is 
  is well known to be the mean $w_t = \E_{P_t}[w]$ of
  non-projected EW with a $\beta(\frac{1}{2}, \frac{1}{2})$ prior 
  and a fixed learning rate $\eta_t = 1$. 
  For the log loss, the mixability gap is $0$. 
To bound the remaining terms in Lemma~\ref{lem:EWRegret}, we choose $Q =
P_{T+1}$, which gives:
\begin{align*}
\sum_{t=1}^T &f_t(w_t)
\leq \E_{P_{T+1}(w)}\left[\sum_{t=1}^T f_t(w)\right] + \kl(P_{T+1}\|P_1) 
= -\ln \E_{P_1(w)}[w^{\sum_{t=1}^T x_t}(1-w)^{T - \sum_{t=1}^T x_t}]\\
&\leq -\ln \max_w \left\{w^{\sum_{t=1}^T x_t}(1-w)^{T - \sum_{t=1}^T
x_t}\right\}
+ \ln(2\sqrt{T})
= \min_w \sum_{t=1}^T f_t(w) + \ln(2\sqrt{T}),
\end{align*}
where the last inequality holds by
\citep[Lemma~9.3]{CesaBianchiLugosi2006}.
\end{example}
For most regret bounds derived from Lemma~\ref{lem:EWRegret} the
structure of the proof remains the same: we need both a bound on the
mixability gap, and a choice for $Q$ for which the expected loss under
$Q$ together with $\kl(Q\|P_1)$ can be related to the loss of a
deterministic comparator.

\section{Linearized Losses}
\label{sec:linearloss}

A standard approach in OCO is to lower-bound the convex losses $f_t$ by
their tangent at $\w_t$, which leads to the following upper bound on the
regret in terms of the linearized surrogate losses $\sloss_t(\w) =
\inner{\w}{\g_t}$, where $\g_t = \nabla f_t(\w_t) =
(g_{t,1},\ldots,g_{t,d})^\top$ is the gradient at~$\w_t$:
\begin{equation}\label{eqn:linearizedloss}
  \sum_{t=1}^T \left(f_t(\w_t) - f_t(\u)\right) \leq \sum_{t=1}^T
  \big(\sloss_t(\w_t) - \sloss_t(\u)\big).
\end{equation}

\subsection{Exponentiated Gradient Plus-Minus as Exponential Weights}\label{sec:EG}

The Exponentiated Gradient Plus-Minus ($\text{EG}^\pm$) algorithm
\citep{KivinenWarmuth} starts with weight vectors $\w_t^- = \w_t^+ =
(1/d,\ldots,1/d) \in \reals^d$, which are updated according to
\begin{align*}
  w_{t+1,i}^+ &= \frac{w_{t,i}^+ e^{-\eta_t \langle \e_i, \g_{t} \rangle}}
    {\sum_{j=1}^d (w_{t,j}^+ e^{-\eta_t \langle \e_j, \g_{t} \rangle} + w_{t,j}^- e^{\eta_t
    \langle \e_j, \g_{t} \rangle})},
  &
  w_{t+1,i}^- &= \frac{w_{t,i}^- e^{\eta_t \langle \e_i, \g_{t} \rangle}}
    {\sum_{j=1}^d (w_{t,j}^+ e^{-\eta_t \langle \e_j \g_{t} \rangle} + w_{t,j}^- e^{\eta_t
    \langle \e_j, \g_{t} \rangle})},
\end{align*}
and predicts by $\w_t \in \{\w: \|\w\|_1 \leq 1\}$ with components
  $w_{t,i} = w_{t,i}^+ - w_{t,i}^-$.

This is readily seen to be the mean $\w_t = \E_{P_t}[\w]$ % 
of EW
(without projections) on the linearized losses
\eqref{eqn:linearizedloss} with a discrete uniform
prior $P_1$ on the standard basis vectors $\e_1,\ldots,\e_d$, which form
the corners of the probability simplex, and their negations
$-\e_1,\ldots,-\e_d$. The regular Exponentiated Gradient algorithm is
recovered by initializing $\w_1^- = (0,\ldots,0)$, which corresponds to
placing prior mass only on $\e_1,\ldots,\e_d$. \Citet{KivinenWarmuth}
also extend the algorithm to scale up the domain by a factor $M > 0$,
which corresponds to a discrete prior on $M\e_1,\ldots,M\e_d$ for EG and
also on $-M\e_1,\ldots,-M\e_d$ for $\text{EG}^\pm$. Hence we may analyze
these methods using Lemma~\ref{lem:EWRegret}, which leads to the
following regret bound for $\text{EG}^\pm$ (see
Appendix~\ref{sec:EGpmproof}):
\begin{theorem}[$\text{EG}^\pm$ as EW]\label{th:EGpm} Suppose $\|\g_t\|_\infty \leq G$ for
all $t$. Then the regret of $\text{EG}^\pm$ for scale factor $M > 0$ and
constant learning rate $\eta_t = \sqrt{\frac{2 \ln(2d)}{TM^2G^2}}$
satisfies
\[
  \regret_T(\u)
    \leq GM\sqrt{2 T\ln(2d)}
    \qquad \text{for all $\u$ such that $\|\u\|_1 \leq M$.}
\]
\end{theorem}

\subsection{Gradient Descent as Exponential Weights}
\label{sec:GD}

The prior of $\text{EG}^\pm$ is adapted to comparators $\u$ with small
$L_1$-norm.
How do we change the prior to favor comparators with small $L_2$-norm? A
natural and computationally efficient choice is to use a Gaussian prior
$P_1 = \normal(\w_1,\sigma^2 \I)$, where $\I$ is the identity matrix.
Then it turns out that all EW distributions $P_t$ are Gaussian with the
Gradient Descent (GD) predictions as their means:
\begin{theorem}[Gradient Descent as EW]\label{thm:GDasEW}
Let $\domainP = \{P :
\E_P[\w] \in \domainw\}$. Then, for Gaussian prior $P_1(\w) =
\normal(\w_1, \sigma^2 \I)$, lazy and
greedy EW with learning rates $\eta_t$ on the linearized
losses \eqref{eqn:linearizedloss} yield Gaussian
distributions $\tilde{P}_t = \normal(\tilde{\w}_t,\sigma^2 \I)$ and $P_t = \normal(\w_t, \sigma^2 \I)$  
with the same covariance as the prior. The means $\tilde{\w}_t$ and
$\w_t$ coincide with lazy and greedy GD (Figure~\ref{alg:GD}), except
that the learning rates in GD are scaled to $\sigma^2 \eta_t$ by the
prior variance $\sigma^2$.
Moreover, Lemma~\ref{lem:EWRegret} directly implies:
\begin{align*}
  \regret_T(\u) ~&\leq ~\frac{\|\u - \w_1\|_2^2}{2\sigma^2\eta_T} + 
      \frac{\sigma^2}{2}\sum_{t=1}^T \eta_{t-1} \|\g_t\|_2^2 & \text{(lazy GD)}\\
  \regret_T(\u) &\leq ~ \frac{\max_t \|\u - \w_t\|_2^2}{2\sigma^2\eta_T} + 
      \frac{\sigma^2}{2}\sum_{t=1}^T \eta_t \|\g_t\|_2^2 & \text{(greedy GD).}
\end{align*}
\end{theorem} 
We note that in this case the parametrization of EW is redundant,
because changing the prior variance $\sigma^2$ has the same effect on
the predictions $\w_t$ and the regret bounds as
scaling all $\eta_t$.
\begin{proof}
$\tilde{P}_t = \normal(\tilde{\w}_t,\sigma^2 \I)$ may be verified analytically from
\eqref{eq:lazy_EW} and \eqref{eq:greedy_EW}. The fact that the
projections $P_t$ onto $\domainP$ preserve Gaussianity with the same
covariance matrix is a property of projecting a member of an exponential
family onto a set of distributions defined by a convex constraint
on their means.
(This follows from Lemma~\ref{lem:mre} in
Appendix~\ref{sec:thm_MD_EW_equivalence} or see
\citep[Lemma~9]{VanErvenKoolen2016} for the Gaussian case.)
The regret bounds follow by taking $Q = \normal(\u,\sigma^2 \I)$, for
which $\kl(Q \| P_t) = \frac{1}{2\sigma^2}\|\u - \w_t\|_2^2$, and
evaluating the mixability gap in closed form.
\end{proof}
\begin{figure}
{% 
\centering \small
\begin{tabular}{| c | c | r |}
\hline
\multicolumn{2}{|l|}{\textbf{Input:} Convex set $\domainw$ and learning rates $\eta_1 \geq  \eta_{2} \geq \ldots \geq
  \eta_T > 0$} \\
\hline
Lazy Gradient Descent & Greedy Gradient Descent \\
\hline
\parbox{7cm}{ \vspace{-0.35cm} \begin{equation*}
\begin{split}
\tilde{\w}_{t+1} &= \textstyle \w_1 - \eta_t \sum_{s=1}^t \g_s \\
\w_{t+1} &= \argmin_{\w \in \domainw} \tfrac{1}{2}\|\w - \tilde{\w}_{t+1}\|^2_2
\end{split}
\end{equation*} \vspace{-0.39cm}} & \parbox{7cm}{ \vspace{-0.35cm} \begin{equation*}
\begin{split}
\tilde{\w}_{t+1} &= \w_t - \eta_t \g_t \\
\w_{t+1} &= \argmin_{\w \in \domainw} \tfrac{1}{2}\|\w - \tilde{\w}_{t+1}\|^2_2
\end{split}
\end{equation*} \vspace{-0.30cm}} \\
\hline
\end{tabular}
}
\caption{The lazy and greedy versions of Gradient Descent}\label{alg:GD}
\end{figure}

\subsection{Mirror Descent and FTRL as EW} 
\label{sec:MD}

The fact that Gradient Descent is an instance of EW raises the question
of whether other instances of MD or
FTRL are special cases of
EW as well. Let $F^*(\w) = \sup_\np \inner{\w}{\np} - F(\np)$ denote the
convex conjugate of $F$, and let $B_{F^*}(\u\|\w) = F^*(\u) - F^*(\w) -
\nabla F^*(\w)^\top (\u - \w)$ denote the corresponding Bregman
divergence. Then MD and FTRL are defined in Figure~\ref{alg:MD} for
Legendre functions $F(\np)$ on $\reals^d$ \citep{CesaBianchiLugosi2006}.
We consider 
exponential families that take the form $\Expfam = \{P_\np \mid \der
P_\np(\w) = e^{\inner{\np}{\w}-F(\np)} \der K(\w), \np \in
\Theta\}$ for a nonnegative \emph{carrier measure} $K$, cumulant
generating function $F(\np) = \ln \int e^{\inner{\np}{\w}} \der K(\w)$
and parameter space $\Theta = \{\np \mid F(\np) < \infty\} \subset
\reals^d$. These are called \emph{regular} if $\Theta$ is an open set. We
then start with the following relation between MD and EW, which is proved in
Appendix~\ref{sec:thm_MD_EW_equivalence}:
\begin{theorem}[Mirror Descent as EW]\label{thm:MDasEW}
  Suppose $F$ is the cumulant generating function of a regular
  exponential family $\Expfam$. Then the lazy and greedy versions of
  MD predict with the means $\w_t = \E_{P_t}[\w]$ of lazy
  and greedy EW on the linearized losses \eqref{eqn:linearizedloss} with
  the same $\eta_t$, prior $P_{\np_1}$ for $\np_1 = \nabla F^*(\w_1)$
  and $\domainP = \{P : \E_P[\w] \in \domainw\}$.
  \label{eq:thm_MD_EW_equivalence}
\end{theorem}
To answer our question, we therefore need to know whether, for any
Legendre function $F^*$, the convex conjugate $(F^*)^* = F$
corresponds to the cumulant generating function of some exponential
family, which means we need to find a corresponding carrier $K$.
Nonconstructive existence of such $K$ has been studied by
\citet[Theorem~6]{banerjee2005}, who show that there is in fact a
bijection between \emph{regular}
Bregman divergences and regular exponential families, where
regular Bregman divergences based on $F^*$ are defined to be those for
which $e^{F(\np)}$ is a continuous, exponentially
convex\footnote{Exponentially convex in the sense of
\citet[Definition~7]{banerjee2005}.} function such
that $\Theta = \{\np \mid F(\np) < \infty\}$ is open and $F$ is strictly
convex.

There is no easy general procedure to construct the corresponding
carrier $K$ for a given Legendre function $F^*$. However, for the Gradient Descent
example from Section~\ref{sec:GD} we see that $F^*(\w) = \frac{1}{2\sigma^2} \|\w\|_2^2$ is the convex conjugate of the cumulant generating
function for $K(\w) = \normal(\0,\sigma^2 \I)$. We also give another example:
\begin{figure}
{% 
\centering \small
\begin{tabular}{| c | c | r |}
\hline
\multicolumn{2}{|l|}{\textbf{Input:} Legendre function $F$, convex set $\domainw$, and learning rates $\eta_1 \geq  \eta_{2} \geq \ldots \geq
  \eta_T > 0$} \\
\hline
FTRL / Lazy Mirror Descent & Greedy Mirror Descent \\
\hline
\parbox{7.2cm}{ \vspace{-0.35cm} \begin{equation*}
\begin{split}
  \tilde{\w}_{t+1} & = \argmin_\w \textstyle\sum_{s=1}^t \inner{\w}{\g_s} + \frac{1}{\eta_t} B_{F^*}(\w\|\w_1) \\
\w_{t+1} & = \argmin_{\w \in \domainw}B_{F^*}(\w\|\tilde{\w}_{t+1}),
\end{split}
\end{equation*} \vspace{-0.35cm} } & \parbox{6.8cm}{ \vspace{-0.35cm} \begin{equation*}
\begin{split}
  \tilde{\w}_{t+1} & = \argmin_\w \inner{\w}{\g_t} + \tfrac{1}{\eta_t} B_{F^*}(\w\|\w_t) \\
\w_{t+1} & = \argmin_{\w \in \domainw}B_{F^*}(\w\|\tilde{\w}_{t+1}).
\end{split}
\end{equation*} \vspace{-0.35cm} } \\
\hline
\end{tabular}
}
\caption{The lazy and greedy versions of Mirror Descent. Lazy MD is
usually called FTRL.}\label{alg:MD}
\end{figure}

\begin{example}[Unnormalized Relative Entropy]
Consider MD with regularization based on the
\emph{unnormalized relative entropy} $B_{F^*}(\w \| \u) = \sum_{i=1}^d
(w_i \ln \frac{w_i}{u_i} - w_i + u_i)$ for $\w,\u \in \reals_+^d$,
which is the Bregman divergence generated by $F^*(\w) = \sum_{i=1}^d w_i
( \ln(w_i) - 1)$ \citep{CesaBianchiLugosi2006}. We have $F(\np) = \sum_{i=1}^d e^{\theta_i} $.
Interestingly, the exponential family with this cumulant generating
function is the set of Poisson distributions,
extended i.i.d.\ to $d$ dimensions. To see this for $d=1$, note that if
we start with the usual parametrization of Poisson, we have
\[
  P_\lambda(w) = e^{-\lambda} \frac{\lambda^w}{w!}
  = \frac{1}{w!} e^{-\lambda + w \ln \lambda}
  \qquad \text{on $w \in \{0,1,2,\ldots$\},}
\]
for which the natural parameter is $\theta = \ln \lambda$ and we see
that the cumulant generating function is $F(\theta) = \lambda =
e^{\theta}$. Thus, EW with the product prior $P_1(\w) = \prod_{i=1}^d
P_{\lambda_i}(w_i)$ corresponds to MD with unnormalized relative entropy, where we need to
set $(\lambda_1,\ldots,\lambda_d) = \exp(\np_1) = \exp(\nabla F^*(\w_1)) = \w_1$
to match the starting point of MD: $\E_{P_1}[\w] = \w_1$. Note that in this case the EW distributions $P_t$ are discrete.

\end{example}

\section{Quadratic Losses}\label{sec:quadlosses}

In this section we assume that the losses 
$f_t$ satisfy quadratic lower bounds:
\begin{equation} \label{eq:quadlowerbound}
f_t(\w) - f_t(\w_t)\geq  \langle \w - \w_t, \g_t \rangle +
\frac{1}{2}(\w - \w_t)^\top \M_t(\w - \w_t) =: \sloss_t(\w),
\end{equation}
where $\M_t$ is a positive semi-definite
matrix. Generalizing the results from Section~\ref{sec:linearloss}, EW
with Gaussian prior on the surrogate loss
$\sloss_t$ yields explicitly computable Gaussian distributions $P_t$
\citep[see also][]{VanErvenKoolen2016,koolen2016blog}:
\begin{figure}
{% 
\hspace{-1cm}
\centering \small
\noindent\resizebox{1.14\linewidth}{!}{% 
\begin{tabular}{| c | c | r |}
\hline
\multicolumn{2}{|l|}{\textbf{Input:} Convex set $\domainw$ and learning rate $\eta > 0$} \\
\hline
Lazy EW Gaussian prior quadratic loss & Greedy EW Gaussian prior quadratic loss \\
\hline
\parbox{8.5cm}{\vspace{-0.35cm} \begin{equation*}
\begin{split}
\Sigma_{t+1}^{-1} & = \Sigma_{t}^{-1} + \eta \M_t \\
\tilde{\w}_{t+1} & = \tilde{\w}_{t} - \eta\Sigma_{t+1}\g_t \\
\w_{t+1} & = \argmin_{\w \in \domainw} (\w - \tilde{\w}_{t+1})^\top \Sigma_{t+1}^{-1}(\w - \tilde{\w}_{t+1})
\end{split}
\end{equation*} \vspace{-0.39cm} } & \parbox{8.5cm}{ \vspace{-0.35cm} \begin{equation*}
\begin{split}
\Sigma_{t+1}^{-1} & = \Sigma_{t}^{-1} + \eta \M_t \\
\tilde{\w}_{t+1} & = \w_{t} - \eta\Sigma_{t+1}\g_t \\
\w_{t+1} & = \argmin_{\w \in \domainw} (\w - \tilde{\w}_{t+1})^\top \Sigma_{t+1}^{-1}(\w - \tilde{\w}_{t+1})
\end{split}
\end{equation*} \vspace{-0.39cm} } \\
\hline
\end{tabular}
}
}
\caption{The means and covariances of both versions of Exponential
Weights with a multivariate normal prior and a constant learning rate
$\eta$ run on the quadratic surrogate loss \eqref{eq:quadlowerbound}}\label{alg:gausquadratic}
\end{figure}
\begin{theorem}\label{th:GaussQloss}
Let $P_1 = \normal(\w_1, \Sigma_1)$. Both versions of the Exponential
Weights algorithm, run on $\sloss_t$ with learning rate $\eta$ and $\domainP = \{P : \E_P[\w] \in \domainw\}$, yield a multivariate normal distribution $P_{t+1}
= \normal(\w_{t+1}, \Sigma_{t+1})$ with mean and covariance matrix given in
Figure~\ref{alg:gausquadratic}.
Furthermore, Lemma~\ref{lem:EWRegret} implies that for all $\u \in \reals^d$ both versions of EW satisfy:
\begin{equation}\label{eq:gausquadraticregret}
   \regret_T(\u) \leq \frac{1}{2\eta} (\w_1 - \u)^\top \Sigma_1^{-1}(\w_1 - \u) + \frac{\eta}{2}\sum_{t=1}^T\g_t^\top\Sigma_{t+1}\g_t.
\end{equation}
\end{theorem}
The proof of Theorem~\ref{th:GaussQloss} in
Appendix~\ref{sec:detailproofgausquadratic} is a straightforward
generalization of Theorem~\ref{thm:GDasEW} for constant
learning rate $\eta_t = \eta$, which is recovered with
$\M_t = \0$. Like in Theorem~\ref{thm:GDasEW}, the
parametrization by $\eta$ and $\sigma^2$ is redundant in that only the
product $\eta\sigma^2$ affects the predictions $\w_t$ or the bound
\eqref{eq:gausquadraticregret}.

\subsection{Gradient Descent: Quadratic Approximation of Strongly Convex Losses}
\label{sec:stronglyconvex}

For $\alpha$-strongly convex loss functions, \eqref{eq:quadlowerbound}
holds with $\M_t = \alpha \I$. The standard approach for these loss
functions is to use greedy Gradient Descent with a time-varying learning rate
$\eta_t = 1/(\alpha t)$ \citep{HazanAgarwalKale2007}. Interestingly, greedy GD with the closely
related choice $\eta_t = 1/(\tfrac{1}{\eta \sigma^2}+\alpha t)$ turns
out to be a special case of greedy EW with \emph{fixed} learning rate
$\eta$ and prior $P_1 = \normal(\0, \sigma^2 \I)$.
Applying Theorem~\ref{th:GaussQloss} results in the following corollary,
proved in Appendix~\ref{sec:proofgausstrong}:
\begin{corollary}\label{lem:gausstongconvexloss}
Suppose $\|\u\|_2 \leq D$ and $\|\g_t\|_2 \leq G$. Then the regret of
both versions of the Exponential Weights algorithm with prior
$\normal(\0, \sigma^2\I)$ and constant learning rate $\eta$, run on the
surrogate loss \eqref{eq:quadlowerbound} with $\M_t = \alpha \I$,
satisfies:
\[
   \regret_T(\u) \leq 
    \frac{G^2}{2\alpha}
    \ln \left(\frac{\tfrac{1}{\eta \sigma^2}+\alpha T}{
    \tfrac{1}{\eta \sigma^2}+\alpha}\right)
  + \frac{G^2}{\tfrac{2}{\eta \sigma^2} + 2\alpha}
  + \frac{D^2}{2\eta\sigma^2}.
\]
\end{corollary}
The standard learning rate and corresponding regret bound for GD
\citep{HazanAgarwalKale2007} correspond to the limiting case $\eta
\sigma^2 \to \infty$. Formally speaking, this case is not covered here,
but for $\eta \to \infty$ EW reduces to Follow-the-Leader (on the
surrogate loss \eqref{eq:quadlowerbound}), and taking $\sigma^2 \to
\infty$ would lead to EW with an \emph{improper prior}, which becomes a
proper EW posterior $P_2$ after one round.

\subsection{Online Newton Step: Quadratic Approximation of Exp-concave Losses}
\label{sec:ons}

For $\alpha$-exp-concave loss functions, \eqref{eq:quadlowerbound} holds
with $\M_t = \beta \g_t \g_t^\top$, where $\beta = \frac{1}{2}
\min\{\frac{1}{4GB}, \alpha\}$, assuming $\|\g_t\|_2 \leq G$ and $B = \max_{\w, \u \in \domainw} \|\w - \u\|_2$ \citep[Lemma 3]{HazanAgarwalKale2007}.
Running Exponential Weights on $\sloss_t(\w)$ with prior $\normal(\0,
\sigma^2\I)$ leads to the Online Newton Step algorithm
\citep{HazanAgarwalKale2007} with the following regret bound, shown in
Appendix~\ref{sec:ONSproof}:
\begin{corollary}\label{lem:gausexpconcaveregret}
Suppose $\|\u\|_2 \leq D$ and $\|\g_t\|_2 \leq G$. Then the regret of
both versions of the Exponential Weights algorithm with prior
$\normal(\0, \sigma^2 \I)$ and learning rate $\eta$, run on the
surrogate loss \eqref{eq:quadlowerbound} with $\M_t = \beta \g_t
\g_t^\top$, satisfies:
\begin{equation}\label{eq:gausexpconcaveregret}
    \regret_T(\u) \leq         \frac{d}{2\beta}\ln\left(1 + \frac{\eta \sigma^2\beta G^2 T}{d} \right)
    + \frac{D^2}{2\eta\sigma^2}.
\end{equation}
\end{corollary}
The results of \citet{HazanAgarwalKale2007} correspond to setting
$\eta \sigma^2 = \beta D^2$, together with some simplifying upper bounds
on \eqref{eq:gausexpconcaveregret}.

\section{Adaptivity by Reduction to Exponential Weights}
\label{sec:adaptivity}

In this section we show how several recent adaptive
 methods in the prediction with experts
setting --- namely iProd \citep{KoolenVanErven2015}, Squint
\citep{KoolenVanErven2015} and a variation of Coin Betting for experts
\citep{OrabonaPal2016} --, whose original analyses seem unrelated at
first sight, can all be viewed as applying exponential weights after
reductions of the original OCO task to various closely related surrogate
OCO tasks. The known regret bounds for these methods are also 
recovered from the reductions upon plugging in regret bounds for EW in the
surrogate tasks.

\subsection{Reduction for iProd}

The experts setting consists of linear losses $f_t(\w) =
\inner{\w}{\g_t}$ over the simplex $\domainw = \{\w: w_i \geq 0,
\sum_{i=1}^d w_i = 1\}$, with $g_{t,i} \in [0,1]$. The instantaneous regret in round $t$ with
respect to expert $i$ is $r_t(i) = f_t(\w_t) - f_t(\e_i)$ and
$\regret_T(i) = \sum_{t=1}^T r_t(i)$ is the total regret. iProd achieves a
second-order regret bound in terms of the data-dependent quantity
$\V_T(i) = \sum_{t=1}^T r_t(i)^2$, which is much smaller than the
worst-case regret in many common cases
\citep{KoolenGrunwaldVanErven2016}.

In the surrogate OCO task for iProd, predictions take the form of joint
distributions $P_t$ on $(\eta,i)$ for $\eta \in [0,1]$ and $i \in
\{1,\ldots,d\}$. These map back to predictions in the original task via
\begin{equation}\label{eqn:wreduction}
  \w_t = \frac{\E_{P_t}[\eta \e_i]}{\E_{P_t}[\eta]},
\end{equation}
which is like the marginal mean of $P_t$ on experts, except that it is
\emph{tilted} to favor larger $\eta$. The surrogate loss in the
surrogate task is
\begin{equation}\label{eqn:iProdSurrogateLoss}
  \sloss_t(\eta,i) = -\ln \left(1 + \eta r_t(i)\right),
\end{equation}
and our aim will be to achieve small \emph{mix-regret} with respect to
any comparator distribution $Q$ on $(\eta,i)$, which we define as
  $\sregret(Q)
    = \sum_{t=1}^T -\ln \E_{P_t}\left[e^{-\sloss_t(\eta,i)}\right]
      - \E_{Q}\Big[\sum_{t=1}^T \sloss_t(\eta,i)\Big]$.
The mix-regret allows exponential mixing of predictions according to
$P_t$ just like for exp-concave losses, so there is no mixability gap to
pay. Exponential weights with constant learning rate $1$ on the losses~$\sloss_t$ therefore achieves $\sregret(Q) \leq
\kl(Q\|P_1)$ for any $Q$.\footnote{This follows e.g.\ from
Lemma~\ref{lem:EWRegret} by subtracting $\sum_t f_t(\w_t)$ on both sides
of \eqref{eqn:EWRegret} and rearranging.} The resulting predictions
$\w_t$ are those of the iProd algorithm. As shown in
Appendix~\ref{sec:iProdProof}, they achieve the following regret bound,
which depends on the surrogate regret of EW:
\begin{theorem}[iProd Reduction to EW]\label{thm:iProd}
  Restrict the domain for $\eta$ to $[0,\half]$. Then any choice
  of~$P_t$ in the surrogate OCO task defined above induces regret
  bounded by
  \begin{equation}\label{eqn:iProdReduction}
    \E_Q[\eta] \sum_{t=1}^T f_t(\w_t) - \E_Q\Big[\eta \sum_{t=1}^T
    f_t(\e_i)\Big]
      \leq \E_Q\Big[\eta^2 \V_T(i)\Big] + \sregret(Q)
    \qquad \text{for any $Q$ on $(\eta,i)$}
  \end{equation}
  in the original prediction with expert advice task.

  In particular, if we use EW in the surrogate OCO task with learning
  rate $1$ and any product prior $P_1 = \gamma \times \pi$ for $\gamma$
  a distribution on $\eta \in [0,\half]$ and $\pi$ a distribution on
  $i$, and we take as comparator $Q = \gamma(\eta \mid \eta \in
  [\etaopt/2,\etaopt]) \times \postpi$ for any $\etaopt \in [0,\half]$
  and distribution $\postpi$ on $i$ that can both depend on all the
  losses, then
  \begin{equation}\label{eqn:iProdSpecial}
    \E_{\postpi}\big[\regret_T(i)\big]
      \leq 2 \etaopt
      \E_{\postpi}[\V_T(i)] + \frac{2}{\etaopt}
      \Big(\kl(\postpi\|\pi) - \ln \gamma([\etaopt/2,\etaopt])\Big).
  \end{equation}
\end{theorem}
Crucially, the algorithm does not need to know $\etaopt$ in advance, but
\eqref{eqn:iProdSpecial} still holds for all $\etaopt$ simultaneously.
To minimize \eqref{eqn:iProdSpecial} in $\etaopt$ we can restrict
ourselves to $\etaopt \geq 1/\sqrt{T}$ without loss of generality, so
that a prior density $\der \gamma(\eta)/\der \eta \propto 1/\eta$ on
$[1/\sqrt{T},1/2]$ achieves $- \ln \gamma([\etaopt/2,\etaopt]) = O(\ln
\ln T)$. After optimizing $\etaopt$, this leads to an adaptive regret
bound of
\begin{equation}\label{eqn:iProdBigOh}
  \E_{\postpi}\big[\regret_T(i)\big]
    = O\left(\sqrt{\E_{\postpi}[\V_T(i)]\Big(\kl(\postpi\|\pi) + \ln \ln T\Big)}\right)
  \qquad \text{for all $\postpi$,} 
\end{equation}
which recovers the results of \citet{KoolenVanErven2015} (see also
\citep{koolen2015blog}).

\subsection{Reduction for Squint}

Running EW with a continuous prior on $\eta$ for the iProd surrogate losses
from \eqref{eqn:iProdSurrogateLoss} requires evaluating a $t$-degree
polynomial in $\eta$ in every round, and therefore leads to $O(T^2)$
total running time. This may be reduced to $O(T \ln T)$ by using a
prior $\gamma$ on an exponentially spaced grid of $\eta$ (as in MetaGrad
\citep{VanErvenKoolen2016}), but in the experts setting even the extra
$\ln T$ factor in run time can be avoided. This is possible by moving
the `prod bound' that occurs in the proof of Theorem~\ref{thm:iProd},
from the analysis into the algorithm by replacing the surrogate loss
from \eqref{eqn:iProdSurrogateLoss} by the slightly larger surrogate
loss
\begin{equation}\label{eqn:SquintSurrogateLoss}
  \sloss_t(\eta,i) = -\eta r_t(i) + \eta^2 r_t(i)^2,
\end{equation}
which turns iProd into Squint. Because this surrogate is quadratic in
$\eta$, it becomes possible to run EW in the resulting surrogate OCO
task and evaluate the resulting integrals over $\eta$ in closed form for
suitable choices of the prior on $\eta$, so that Squint has $O(T)$ run
time (see \citet{KoolenVanErven2015} for a detailed discussion of the
choice of prior). Moreover, as shown in Appendix~\ref{sec:SquintProof},
it satisfies exactly the same guarantees as iProd.

\subsection{Reduction for Coin Betting}

If we are willing to give up on second-order bounds, but still want to
learn $\eta$, then there is another way to obtain an algorithm with
$O(T)$ run time by bounding the iProd surrogate loss, which leads to a
variant of the Coin Betting algorithm for experts of
\citet{OrabonaPal2016}. Our presentation and analysis are very
different from \citep{OrabonaPal2016}, but we obtain exactly the same
regret bound for essentially the same algorithm, and we can explain some
design choices that required clever insights by \citet{OrabonaPal2016},
as natural consequences of running EW in the surrogate OCO task that we
end up with.

The idea is to split the learning of $\eta \in [0,1]$ and $i$ into
separate steps: for each $i$, we restrict $P_t(\eta \mid i)$ to be a
point mass on some $\eta_t^i$, and we will choose $\eta_t^i$ to achieve
small regret for the surrogate loss
\[
  \sloss_t^i(\eta)
    = -\frac{1+r_t(i)}{2} \ln \frac{1+\eta}{2}
      -\frac{1-r_t(i)}{2} \ln \frac{1-\eta}{2}
      -\ln 2,
\]
which upper bounds \eqref{eqn:iProdSurrogateLoss} by convexity of the
negative logarithm. We then plug in the choices of $\eta_t^i$ in
\eqref{eqn:iProdSurrogateLoss} and learn $i$ for the resulting surrogate
losses $\altsloss_t(i) = -\ln(1+\eta_t^i r_t(i))$.
For $\eta \in [0,1]$ and $\postpi$ a distribution on $i$, let
\begin{align*}
  \sregret_T^i(\eta)
    &= \sum_{t=1}^T \sloss_t^i(\eta_t^i) - 
       \sum_{t=1}^T \sloss_t^i(\eta),
  &\altsregret_T(\postpi)
    &= \sum_{t=1}^T -\ln \E_{i \sim P_t}\left[e^{-\altsloss_t(i)}\right]
      - \E_{\postpi}\Big[\sum_{t=1}^T \altsloss_t(i)\Big]
\end{align*}
be the mix-regret in the two surrogate OCO tasks. (Notice that in
$\sregret_T^i$ the mix-regret has collapsed to the ordinary regret,
because we are restricting ourselves to play point masses on $\eta$.)
Also let $\regret_T^+(i) = \max\{\regret_T(i),0\}$ be the nonnegative
part of the regret, and define $\bernoullikl(x\|y) = x\ln \frac{x}{y} +
(1-x)\ln \frac{1-x}{1-y}$ to be the Kullback-Leibler divergence between
two Bernoulli distributions, which satisfies $\bernoullikl(x\|y) \geq
2(x-y)^2$ by Pinsker's inequality. Then this reduction gives the
following regret bound, proved in Appendix~\ref{sec:coinbettingproof}:
\begin{theorem}[Coin Betting Reduction to EW]\label{thm:coinbetting}
  Any choice of distributions $P_t$ on $i$ and learning rates $\eta_t^i$
  in the surrogate OCO task defined above induces regret bounded by
  \begin{equation}\label{eqn:coinbettinggeneral}
    \E_{\postpi}\left[\bernoullikl\left(\thalf + \tfrac{\regret_T^+(i)}{2T} \|
    \thalf\right)\right]
      \leq \tfrac{1}{T}
      \Big(\E_{\postpi}\left[\sregret_T^i\left(\tfrac{\regret_T^+(i)}{T}\right)\right] +
      \altsregret_T(\postpi)\Big)
    \qquad \text{for any $\postpi$ on $i$}
  \end{equation}
  in the original prediction with expert advice task.

  In particular, if we use EW with learning rate $1$ and prior $\pi$ on
  $i$ for the losses $\altsloss_t$, and for the losses $\sloss_t^i$ we
  let $\eta_t^i$ be the mean of lazy EW with learning rate $1$ and with
  prior on $\eta \in [-1,+1]$ such that $\frac{1+\eta}{2}$ has a
  beta-distribution $\beta(a,a)$ with $a=\frac{T}{4}+ \half$ and with
  projections onto $\domainP = \{P \mid \E_P[\eta] \in [0,1]\}$, then 
  \begin{equation}\label{eqn:coinbettingEW}
    \E_{\postpi}\left[\regret_T(i)\right]
      \leq 
      \sqrt{3T\left(\kl(\postpi\|\pi) + 3\right)}
    \qquad \text{for any $\postpi$ on $i$.}
  \end{equation}
\end{theorem}
Compared to \eqref{eqn:iProdBigOh}, \eqref{eqn:coinbettingEW} avoids a
$\ln \ln T$ term, but it has lost the benefits of the second-order
factor $\E_{\postpi}[\V_T(i)] \leq T$. This may be explained by its upper
bound $\sloss_t^i(\eta) \geq \sloss_t(\eta,i)$, which is tight only in
the extreme case that
$r_t(i) \in \{-1,+1\}$.

\paragraph{The Resulting Coin Betting Algorithm}

EW on the losses $\sloss_t^i$ with the (conjugate) $\beta(a,a)$ prior is
a generalization of the Krichevsky-Trofimov estimator (see
Example~\ref{ex:KT}) and its mean has the closed form
$\frac{\regret_{t-1}(i)}{t-1+2a}$. Lazily projecting onto $\domainP$ then
simply amounts to clipping at $0$ (by convexity of KL-divergence in its
first argument, which implies that the constraint $\E_P[\eta] \geq 0$
will be satisfied with equality when we project from a distribution with
negative mean). This means that
  $\eta_t^i = \max\left\{\frac{\regret_{t-1}(i)}{t-1+2a},0\right\}$.
By \eqref{eqn:wreduction} the Coin Betting algorithm from the theorem
predicts with weights $w_{t,i}$ obtained by normalizing the unnormalized
weights
  $\altw_{t,i} = \altp_t(i)\eta_t^i$,
where $\altp_t(i)$ is the unnormalized probability $P_t(i)$ of EW on the
losses $\altsloss_t$, which recursively satisfies
\begin{align*}
  \altp_t(i) &:= \pi(i)\prod_{s=1}^{t-1}(1+\eta_s^i r_s(i))
             = \altp_{t-1}(i) + \altw_{t-1,i} r_{t-1}(i)
             = \ldots = \pi(i) + \sum_{s=1}^{t-1} \altw_{s,i} r_s(i).
\end{align*}
Interestingly, \citet{OrabonaPal2016} interpret the unnormalized EW
probabilities $\altp_t(i)$ as the \emph{Wealth} for expert $i$ that is
achieved by a gambler.

The interpretation in Theorem~\ref{thm:coinbetting} explains three
design choices by \citet{OrabonaPal2016}: first, their choice of
potential function, which naturally arises in our proof when we bound
the regret $\sregret_T^i(\regret_T^+(i)/T)$ for EW using
Lemma~\ref{lem:EWRegret}. Second, the choice for $a$, which in the
original analysis comes from defining a shifted potential function, is
simply specifying a prior with most mass in a region of order
$1/\sqrt{T}$ around $\eta = 0$. And, third, the clipping of the
unnormalized weights $\altw_{t,i}$ to $0$ when $\regret_{t-1}(i) < 0$, which
in our presentation happens automatically because the learning rate
$\eta_t^i$ is projected to be $0$ if it would otherwise become negative.
Defining a prior on positive learning rates directly would be possible
in theory, but not with a conjugate prior, so the computational
efficiency of the algorithm is made possible by the projections.

There is one slight difference between the algorithm we obtain here and
the original Coin Betting algorithm of \citet{OrabonaPal2016}: in the
original method the instantaneous regrets are clipped to
$\max\{r_t(i),0\}$ when $\regret_{t-1}(i) < 0$, which our method does not do.
Apparently there is some amount of freedom in the design of this type of
algorithm.

\section{Online Linear Optimization with Bandit Feedback}
\label{sec:bandits}

A benefit of the EW interpretation of MD is 
that it opens up the possibility of sampling from
the EW posterior distribution instead of playing the mean. 
Here we show how this option can be leveraged
to obtain an algorithm for online linear optimization with 
bandit feedback \citep{Dani_et_al_2007,abernethy2008}, 
which recovers the best known rate $O(d \sqrt{T \ln T})$. 
A proof of this fact has already been outlined by \citet{bubeck2014},
but here we fill in the technical details.

The linear bandit setting consists of linear losses $f_t(\w) =
\inner{\w}{\g_t} \in [-1,+1]$, but instead of seeing the vectors
$\g_t$ we only observe $f_t(\w_t)$ for the algorithm's choice
$\w_t$. The algorithm can randomize its choice $\w_t$,
and $\g_t$ is fixed before the outcome of this randomization. The
goal is to minimize the expected regret $\E[\regret_T(\u)]$, where the
expectation is with respect to the algorithm's randomness.

We consider the EW algorithm with fixed learning rate $\eta$
and uniform prior distribution $P_1$ over~$\domainw$. In each round $t$,
after observing $f_t(\w_t) = \inner{\w_t}{\g_t}$, the algorithm
constructs a random, unbiased estimate $\tilde{\g}_t$ of the loss vector
$\g_t$ and uses this estimate to update $P_t$ to $P_{t+1}$. It is easy to
verify that, for each $t$, $P_t$ is a member of the exponential family
with cumulant generating function $F(\np) = \ln \int_{\domainw}
e^{\inner{\w}{\np}} \; \mathrm{d} \w$. At trial $t$, the algorithm
samples $\w_t \sim Q_t$, where $Q_t = (1-\gamma) P_t + \gamma R$ is a
mixture of the EW distribution $P_t$ and a fixed ``exploration''
distribution $R$, chosen to be \emph{John's exploration}
\citep{Bubeck12}. Using that the convex conjugate of $F$ is a universal
$O(d)$-self concordant barrier on $\mathcal{\domainw}$
\citep{bubeck2014}, it can be shown that, when~$\eta$ and~$\gamma$ are
appropriately chosen, this algorithm achieves expected regret of order
$O(d\sqrt{T \ln T})$ (see Appendix~\ref{sec:appendix_bandits}).

It is interesting to compare with the \emph{SCRiBLe}
algorithm \citep{Scrible}, which replaces EW by MD. By the
results of Section~\ref{sec:MD}, this is an essentially equivalent
approach, except that SCRiBLe employs a sampling strategy based on the
spectrum of the Hessian of $F^*$, without reference to the EW
distribution, and achieves a regret bound that is suboptimal in~$d$.
This shows that the EW interpretation of MD is clearly beneficial in the
bandit setting.

\section{Discussion}

We conclude with several remarks: first, we point out that there may be
computational reasons to avoid defining the prior directly on the domain
$\domainw$ of interest: as shown for instance in Sections~\ref{sec:GD}
and~\ref{sec:quadlosses}, defining a Gaussian prior on all of $\reals^d$
and then projecting the mean onto $\domainw$ can be computationally more
efficient. In the context of sampling from the EW distribution,
discussed in Section~\ref{sec:bandits}, this might also make sense if we
project onto the alternative (smaller) set of distributions $\domainP =
\{P \mid P(\domainw) = 1\} \subset \{P \mid \E_P[\w] \in \domainw\}$
that are supported on~$\domainw$, which amounts to conditioning on
$\domainw$. Second,
there seems
to be a discrepancy between the body of work for the log loss cited in
the introduction, which strongly suggests using Jeffreys' prior, and the
uniform prior suggested in Section~\ref{sec:bandits} in the context of
the universal barrier.

\clearpage
\acks{The authors would like to thank Wouter Koolen for extensive
discussions underlying Theorems~\ref{thm:GDasEW}, \ref{th:GaussQloss},
\ref{thm:iProd} and~\ref{thm:Squint}. A precursor to
Theorem~\ref{thm:MDasEW} previously appeared in Van der Hoeven's master's
thesis \citep{vanderHoeven2016}. He was supported by the Netherlands
Organization for Scientific Research (NWO grant TOP2EW.15.211).
Kot{\l}owski was supported by the Polish National Science Centre (grant
no.\ 2016/22/E/ST6/00299).}

\DeclareRobustCommand{\VAN}[3]{#3} % 

\bibliography{cew}

\DeclareRobustCommand{\VAN}[3]{#2} % 

\appendix

\section{Proof of Lemma~\ref{lem:EWRegret} from Section~\ref{sec:EW}}
\label{sec:EWRegretProof}

In the following we make use of the generalized Pythagorean inequality
for Kullback-Leibler divergence \citep{Csiszar1975}:
for $P_{t} = \argmin_{P \in \domainP}~ \kl(P\|\tilde{P}_{t})$ and any $Q \in \domainP$:
\begin{equation}\label{eq:pineqKL}
\kl(Q\|\tilde{P}_t) \geq \kl(Q\|P_t) + \kl(P_t\|\tilde{P}_t).
\end{equation}
For greedy EW we have
\begin{align*}
  \frac{1}{\eta_t} \big(\kl(Q\|P_t) - \kl(Q\|P_{t+1}) \big)
  & \geq \frac{1}{\eta_t} \left( \kl(Q\|P_t) - \kl(Q\|\tilde{P}_{t+1})\right) \qquad
    & \text{(from \eqref{eq:pineqKL})} \\
    & = \E_{Q}[f_t(\w)] - \frac{1}{\eta_t} \ln\E_{P_t}\Big[e^{-\eta_{t}f_t(\w)}\Big] \qquad
    & \text{(from \eqref{eq:greedy_EW})}
\end{align*}
in any trial $t$. Summing over trials gives:
\begin{align*}
  \sum_{t=1}^T \E_{Q}[f_t(\w)] - \frac{1}{\eta_t} &\ln\E_{P_t}\Big[e^{-\eta_{t}f_t(\w)}\Big]
\leq \sum_{t=1}^T \frac{1}{\eta_t} \big(\kl(Q\|P_t) - \kl(Q\|P_{t+1}) \big) \\
&= \frac{1}{\eta_1} \kl(Q\|P_1)
- \frac{1}{\eta_T} \kl(Q\|P_{T+1}) +
\sum_{t=2}^T \kl(Q\|P_t) \left(\frac{1}{\eta_t}-\frac{1}{\eta_{t-1}}\right) \\
&\leq \frac{1}{\eta_1} \kl(Q\|P_1) + \max_{t=2,\ldots,T} \kl(Q\|P_t) 
\left(\frac{1}{\eta_T}-\frac{1}{\eta_1}\right).
\end{align*}
Rearranging the terms and adding $\sum_{t=1}^T f_t(\w_t)$ on both sides results
in \eqref{eqn:EWRegret_greedy}.

We now proceed with the proof of lazy EW,
starting from:
\begin{align}
  -\frac{1}{\eta_{t-1}}\ln \E_{P_t}[e^{-\eta_{t-1} f_t(\w)}]  
  &=\min_P\bigg\{ \E_P[f_t(\w)] + \frac{1}{\eta_{t-1}}\kl(P\|P_t) \bigg\} \nonumber \\
  & \leq \E_{P_{t+1}}[f_t(\w)] + \frac{1}{\eta_{t-1}}\kl(P_{t+1}\|P_t) \nonumber \\
  & \leq  \E_{P_{t+1}}[f_t(\w)] + \frac{1}{\eta_{t-1}}\kl(P_{t+1}\|\tilde{P}_t) - \frac{1}{\eta_{t-1}}\kl(P_{t}\|\tilde{P}_t),
\label{eq:some_intermediate_bound}
\end{align}
where the last inequality is from the Pythagorean inequality
\eqref{eq:pineqKL} applied with $Q=P_{t+1}$. By \eqref{eq:lazy_EW}:
\[
  \ln \frac{\der \tilde{P}_t(\w)}{\der P_{1}(\w)} = - \eta_{t-1}
  \sum_{s=1}^{t-1} f_s(\w)
  - \ln \E_{P_1}\left[e^{-\eta_{t-1} \sum_{s=1}^{t-1} f_s(\w)}\right],
\]
which gives:
\begin{align*}
  \frac{1}{\eta_{t-1}}\kl(P_{t+1}\|\tilde{P}_t) - \frac{1}{\eta_{t-1}}\kl(P_{t}\|\tilde{P}_t)
  &=\frac{1}{\eta_{t-1}}\kl(P_{t+1}\|P_1) - \frac{1}{\eta_{t-1}}\kl(P_{t}\|P_1) \\
  &+ \E_{P_{t+1}}\bigg[\sum_{s=1}^{t-1} f_s(\w) \bigg] 
  - \E_{P_t}\bigg[\sum_{s=1}^{t-1} f_{s}(\w)\bigg].
\end{align*}
Plugging this into \eqref{eq:some_intermediate_bound} and using $\eta_{t} \leq \eta_{t-1}$ results in:
\begin{align*}
  -\frac{1}{\eta_{t-1}}\ln \E_{P_t}[e^{-\eta_{t-1} f_t(\w)}]  
  &\leq \frac{1}{\eta_t}\kl(P_{t+1}\|P_1) - \frac{1}{\eta_{t-1}}\kl(P_{t}\|P_1) \\
  &+ \E_{P_{t+1}}\bigg[\sum_{s=1}^{t} f_s(\w) \bigg] 
  - \E_{P_t}\bigg[\sum_{s=1}^{t-1} f_{s}(\w)\bigg].
\end{align*}
Summing over trials 
makes the terms on the right-hand side telescope and gives:
\begin{align*}
  \sum_{t=1}^T -\frac{1}{\eta_{t-1}}\ln \E_{P_t}[e^{-\eta_{t-1} f_t(\w)}]  
  &\leq \frac{1}{\eta_{T}} \kl(P_{T+1}\|P_1) + \E_{P_{T+1}}\bigg[ \sum_{t=1}^T f_t(\w) \bigg] \\
  &=\min_{P \in \domainP} \left\{\E_{P}\bigg[ \sum_{t=1}^T f_t(\w) \bigg]
+ \frac{1}{\eta_{T}} \kl(P\|P_1) \right\} \\
&\leq \E_{Q}\bigg[ \sum_{t=1}^T f_t(\w) \bigg] + \frac{1}{\eta_T} \kl(Q\|P_1),
\end{align*}
where the equality expresses an equivalent way to define lazy EW\@. 
Rearranging the terms 
and adding $\sum_{t=1}^T f_t(\w_t)$ on both sides results
in \eqref{eqn:EWRegret}.

\section{Proof of Theorem~\ref{th:EGpm}}\label{sec:EGpmproof}

\begin{proof}
Rather than scaling canonical vectors $\e_i$, $i=1,\ldots,d$ and the comparator
$\u$ by $M$, we scale the
loss vectors by defining $\g_t' = M \g_t$, so that the losses remain the same:
$\inner{\e_i}{\g_t'} = \inner{M \e_i}{\g_t}$ for all $i$ and all $t$.
Let $\w_1 = (\w_1^+, \w_1^-)$, and let $\w_t^+$, $\w_t^-$ be the result of running EG plus-minus on $\g_t'$. For any $\u$ with $\sum_{i = 1}^{2d} u_i = 1$ and $u_i \geq 0$ invoking Lemma~\ref{lem:EWRegret} gives:
\begin{align}
  \sum_{t=1}^T\langle \w_t - \u, \g_t' \rangle  
  &\leq \frac{1}{\eta}\kl(\u\|\w_1) \nonumber  \\
  &+ \sum_{t=1}^T \langle \w_t^+, \g_t' \rangle - \langle \w_t^-, \g_t' \rangle
    + \frac{1}{\eta}\ln \Big(\sum_{i=1}^d (w_{t,i}^+ e^{-\eta_t \langle \e_i, \g_{t}' \rangle} + w_{t,i}^- e^{\eta_t
    \langle \e_i, \g_{t}' \rangle})\Big).
\label{eq:lem1EG}
\end{align}
The first term on the right-hand side of \eqref{eq:lem1EG} 
can be bounded by: $\max_{\u: \sum_{i = 1}^{2d} u_i = 1, ~u_i \geq 0}
\kl(\u\|\w_1) = \ln(2d)$. To bound the second term on the right-hand side of
\eqref{eq:lem1EG}, 
we make use of Hoeffding's Lemma \citep[Lemma A.1]{CesaBianchiLugosi2006}, 
which together with $|\inner{\e_i}{\g_t'}| \leq MG$ gives:
\[
 \sum_{t=1}^T \langle \w_t^+, \g_t' \rangle - \langle \w_t^-, \g_t' \rangle +
 \frac{1}{\eta}\ln \Big(\sum_{i=1}^d (w_{t,i}^+ e^{-\eta_t \langle \e_i, \g_{t}' \rangle} + w_{t,i}^- e^{\eta_t
 \langle \e_i, \g_{t}' \rangle})\Big) \leq \frac{\eta M^2 G^2}{2}.
\]
Summing over trials results in a bound on the regret:
\[
\sum_{t=1}^T\langle \w_t - \u, \g_t' \rangle \leq \frac{\ln(2d)}{\eta} + \eta
\frac{T M^2G^2}{2}.
\]
Plugging in the optimal $\eta = \sqrt{\frac{2 \ln(2d)}{T M^2 G^2}}$ yields the desired result. 
\end{proof}

\section{Proof of Theorem \ref{eq:thm_MD_EW_equivalence}}
\label{sec:thm_MD_EW_equivalence}

Before proving the theorem, we need two lemmas: 
\begin{lemma}[\cite{banerjee2005, nielsen2010}]\label{th:KLisB}
The KL divergence between two members, $P$ and $Q$, of the same regular
exponential family $\Expfam$ with cumulant generating function $F$ can
be expressed by the Bregman divergence between their natural parameters,
$\np_P$ and $\np_Q$, or their expectation parameters, $\Mu_P$ and
$\Mu_Q$. The first Bregman divergence is generated by the cumulant
generating function $F$ and the second Bregman divergence is generated
by the convex conjugate of the cumulant generating function $F^*$:
\[
\kl(P\|Q) = B_F(\np_Q\|\np_P) = B_{F^*}(\Mu_P\|\Mu_Q).
\]
\end{lemma}

\begin{lemma}\label{lem:mre}\citep[Theorem 3.1.4]{ihara1993information}
Let $\Mu$ be arbitrary and define $\domainP =\{P:\E_P[\w] = \Mu\}$.
Then, for any member $Q$ of an exponential family $\Expfam$,
\[
\underset{P \in \domainP}{\min}~\kl(P\|Q)
\]
is achieved by $P \in \Expfam$ such that $\E_{P}[\w] = \Mu$, provided such a $P$ exists.
\end{lemma}

\begin{proof}[of Theorem \ref{eq:thm_MD_EW_equivalence}]
Let $\w_t$ be the weights produced by the greedy version of MD\@. Then
\begin{align*}
  \min_{P \in \domainP} \left\{\E_P[\inner{\w}{\g_t}] + \frac{1}{\eta_t} \kl(P\|P_t)\right\} 
  &= \min_{\Mu \in \domainw} \; \min_{P \,:\, \E_P[\w]=\Mu} \left\{\E_P[\inner{\w}{\g_t}] + \frac{1}{\eta_t} \kl(P\|P_t)\right\} \\
  &= \min_{\Mu \in \domainw} \;\;\min_{P \in \Expfam \,:\, \E_P[\w]=\Mu} \left\{ \inner{\Mu}{\g_t} + \frac{1}{\eta_t}
\kl(P\|P_t) \right\}, 
\end{align*}
where in the second step we can restrict to minimization over $\Expfam$ by
Lemma~\ref{lem:mre}. Introducing the short-hand notation $\Mu_P =
\E_{P}[\w]$, we thus get for the greedy version of EW:
\[
  P_{t+1} = \argmin_{P \in \Expfam : \Mu_P \in \domainw}
  \left\{ \inner{\Mu_P}{\g_t} + \frac{1}{\eta_t}
\kl(P\|P_t) \right\} = 
\argmin_{P \in \Expfam : \Mu_P \in \domainw}
\left\{ \inner{\Mu_P}{\g_t} + \frac{1}{\eta_t} B_{F^*}(\Mu_P\| \Mu_{P_t}) \right\},
\]
where we used Lemma~\ref{th:KLisB}.
But the last expression coincides with the definition of the greedy MD weight update, and
since it applies to all $t$, we have $\Mu_{P_{t+1}} = \w_{t+1}$ for all $t$, provided
$\Mu_{P_1} = \w_1$ (which holds by assumption).
An analogous argument can be made to show the equivalence of the lazy
versions of MD and EW.
\end{proof}

\section{Proofs for Section~\ref{sec:quadlosses}}

\subsection{Proof of Theorem~\ref{th:GaussQloss}}\label{sec:detailproofgausquadratic}
\begin{proof}
$\tilde{P}_t = \normal(\tilde{\w}_t, \Sigma_t)$ may be verified analytically from \eqref{eq:lazy_EW} and \eqref{eq:greedy_EW}. The fact that projections $P_t$ onto $\domainP$ preserve Gaussianity with the same covariance matrix follows from Lemma 9 in \citet{VanErvenKoolen2016}. Lemma~\ref{lem:EWRegret} gives a bound on the regret w.r.t.\ randomized forecaster $Q = \normal(\u, \Sigma_Q)$:
\[
    \sum_{t=1}^T f_t(\w_t) - \sum_{t=1}^T \E_{Q}[f_t(\w)] \leq
    \frac{1}{\eta}\kl(Q\|P_1) + \sum_{t=1}^T f_t(\w_t) + \frac{1}{\eta}
    \ln\E_{P_t}\Big[e^{-\eta f_t(\w)}\Big].
\]
The KL divergence between two Gaussians is given by \citep[Theorem 1.8.2]{ihara1993information}: 
\[
  \kl(Q\|P_1) = \frac{1}{2}(\ln\bigg(\frac{\det(\Sigma_Q)}{\det(\Sigma_1)}\bigg) +
    \Tr(\Sigma_Q \Sigma_1^{-1}) + (\u - \w_1)^\top\Sigma_1^{-1}(\u -
    \w_1) - d.
\]
The mixability gap can be evaluated in closed form by calculating the Gaussian integral:
\[
    \ln \E_{P_t}\Big[e^{\eta ( f_t(\w_t) - f_t(\w))}\Big] =
    \frac{\eta^2}{2}\g_t^\top\Sigma_{t+1}\g_t -
    \frac{1}{2}\ln\bigg(\frac{\det(\Sigma_t)}{\det(\Sigma_{t+1})}\bigg).
\]
Also, the expectation of the instantaneous regret can be computed exactly:
\[
    f_t(\w_t) - \E_{Q}[f_t(\w)] = f_t(\w_t) - f_t(\u) - \frac{1}{2}\Tr(\Sigma_Q \M_t).
\]
Summing the above over the trials, we get the following upper bound on the regret:
\begin{align*}
  \sum_{t=1}^T f_t(\w_t) - \sum_{t=1}^T f_t(\u) 
  &\leq \frac{\ln \Big( \frac{\det(\Sigma_{T+1})}{\det(\Sigma_Q)}\Big) +
\Tr(\Sigma_Q\Sigma_{T+1}^{-1}) - d + (\w_1 - \u)^\top \Sigma_1^{-1}(\w_1 - \u)}{2 \eta} \\
&+ \eta \sum_{t=1}^T\g_t^\top\Sigma_{t+1}\g_t,
\end{align*}
which holds for all $\Sigma_Q$. By plugging in the optimal value 
$\Sigma_Q =\Sigma_{T+1}$, the bound simplifies to:
\[
    \sum_{t=1}^T f_t(\w_t) - \sum_{t=1}^T f_t(\u) \leq \frac{1}{2\eta} (\w_1 -
    \u)^\top \Sigma_1^{-1}(\w_1 - \u) + \frac{\eta}{2}\sum_{t=1}^T\g_t^\top\Sigma_{t+1}\g_t,
\]
which concludes the proof.
\end{proof}

\subsection{Proof of Corollary~\ref{lem:gausstongconvexloss}}\label{sec:proofgausstrong}
\begin{proof}
Using Theorem~\ref{th:GaussQloss} gives:
\[
\begin{split}
    \sum_{t=1}^T f_t(\w_t) - \sum_{t=1}^T f_t(\u) & \leq \frac{1}{2\eta\sigma^2}\|\u\|^2_2 + \frac{\eta }{2} \sum_{t=1}^T  \frac{1}{\frac{1}{\sigma^2} + \alpha \eta t}\|\g_t\|_2^2 \\
    & \leq \frac{1}{2\eta\sigma^2} D^2 + \frac{\eta}{2} G^2 \sum_{t=1}^T\frac{1}{\frac{1}{\sigma^2} + \alpha \eta t}\\
    & \leq \frac{1}{2\eta\sigma^2} D^2 + \frac{\eta G^2}{2(\tfrac{1}{\sigma^2} + \alpha \eta)}  + \frac{\eta}{2} G^2 \int_{1}^T\frac{1}{\frac{1}{\sigma^2} + \alpha \eta t}dt \\
    & = \frac{1}{2\eta\sigma^2} D^2 + \frac{G^2}{2(\tfrac{1}{\eta \sigma^2} + \alpha)}  +  \frac{G^2}{2\alpha}
    \big(\ln(\tfrac{1}{\eta \sigma^2}+\alpha T) -
    \ln(\tfrac{1}{\eta \sigma^2}+\alpha)\big),
\end{split}    
\]
which was to be shown.
\end{proof}

\subsection{Proof of Corollary~\ref{lem:gausexpconcaveregret}}\label{sec:ONSproof}
\begin{proof}
Using Theorem~\ref{th:GaussQloss} gives:
\begin{equation}\label{eq:step1ONS}
\begin{split}
    \regret_T(\u) & \leq \frac{D^2}{2\eta\sigma^2} + \frac{\eta}{2}\sum_{t=1}^T\g_t^\top\Sigma_{t+1}\g_t.
\end{split}    
\end{equation}
We start by bounding the second term on the right-hand side of \eqref{eq:step1ONS}. Using Lemma 11.11 from \citet{CesaBianchiLugosi2006} 
and the basic inequality $1-x \leq -\ln x$, we bound:
\[
  \eta \beta \g_t^\top \Sigma_{t+1} \g_t
    = 1- \frac{\det(\Sigma_t^{-1})}{\det(\Sigma_{t+1}^{-1})} 
  \leq \ln\frac{\det(\Sigma_{t+1}^{-1})}{\det(\Sigma_{t}^{-1})}, 
\]
which after summing over trials gives:
\begin{align*}
  \sum_{t=1}^T \eta \beta \g_t^\top \Sigma_{t+1} \g_t
  &\leq \ln\frac{\det(\Sigma_{T+1}^{-1})}{\det(\Sigma_{1}^{-1})}
  = \ln \det \big(\I + \eta \sigma^2 \beta \sum_{t=1}^T \g_t
  \g_t^\top\big)\\
  &= \sum_{i=1}^d \ln (1 + \lambda_i)
  \leq d \ln\bigg(1 + \frac{\eta\sigma^2 \beta G^2 T}{d}\bigg),
\end{align*}
where $\lambda_1,\ldots,\lambda_d$ are the eigenvalues of $ \eta\sigma^2
\beta \sum_{t=1}^T \g_t \g_t^\top$, and the
last inequality follows by maximizing under the constraint that
$\sum_i \lambda_i = \Tr( \eta\sigma^2
\beta \sum_{t=1}^T \g_t \g_t^\top) \leq \sigma^2 \eta \beta G^2 T$. As
discussed by
\citet[proof and
discussion of Theorem~11.7]{CesaBianchiLugosi2006}, the maximum is
achieved when $\lambda_i = \sigma^2 \eta \beta G^2 T / d$ for all $i$.

All together we find:
\[
    \regret_T(\u) \leq \frac{D^2}{2\eta\sigma^2} +
        \frac{d}{2\beta}\ln\left(1 + \frac{\eta \sigma^2\beta G^2 T}{d} \right),
\]
which was to be shown.

\end{proof}

\section{Proofs for Section~\ref{sec:adaptivity}}

\subsection{Proof of Theorem~\ref{thm:iProd}}
\label{sec:iProdProof}

Abbreviate $m_t(P) = -\ln
\E_{P}\left[e^{-\sloss_t(\eta,i)}\right]$ and define the
potential $\Phi_T = e^{-\sum_{t=1}^T m_t(P_t)}$. Then
$\Phi_T = \Phi_{T-1} = \cdots = \Phi_0 = 1$ since
\[
  \Phi_T - \Phi_{T-1}
    ~=~ e^{-\sum_{t=1}^{T-1} m_t(P_t)} \E_{P_T}\big[\eta r_T(i)\big] 
    ~=~ 0,
\]
where the last identity holds for any loss vector $\g_t$ by the
definition of $\w_T$. For any comparator $Q$ on $(\eta,i)$, it follows
that
\[
  0 = \sum_{t=1}^T m_t(P_t)
    = \sum_{t=1}^T \E_{Q}[\sloss_t(\eta,i)] + \sregret(Q)\\
    \leq \sum_{t=1}^T \E_{Q}[-\eta r_t(i) + \eta^2 r_t(i)^2] +
    \sregret(Q),
\]
where the last inequality is an application of the `prod-bound'
$-\ln(1+x) \leq -x + x^2$ with $x = \eta r_t(i)$, which holds for any
$x \geq -\half$ \citep[Lemma~1]{CesaBianchiMansourStoltz2007}. The
result \eqref{eqn:iProdReduction} is a direct consequence, and
\eqref{eqn:iProdSpecial} follows upon bounding $\E_Q[\eta] \geq
\etaopt/2$ and $\E_Q[\eta^2] \leq \etaopt^2$ and plugging in that
$\sregret(Q) \leq \kl(Q\|P_1) = \kl(\postpi\|\pi) - \ln
\gamma([\etaopt/2,\etaopt])$ for EW.

\subsection{Proof of Theorem~\ref{thm:Squint}}
\label{sec:SquintProof}

\begin{theorem}[Squint Reduction to EW]\label{thm:Squint}
  The exact same statement as in Theorem~\ref{thm:iProd} also holds when
  we replace the surrogate loss \eqref{eqn:iProdSurrogateLoss} by
  \eqref{eqn:SquintSurrogateLoss}.
\end{theorem}
Thus \eqref{eqn:iProdBigOh} also holds, and we recover the results of
\citep{KoolenVanErven2015} for Squint.

\begin{remark}
  The Metagrad algorithm \citep{VanErvenKoolen2016} is similar to Squint
  on a continuous set of experts indexed by $\w \in \reals^d$ with
  losses $f_t(\w) = \w^\top\g_t$, and the analysis of
  Theorem~\ref{thm:Squint} can be extended to handle this case.
\end{remark}

\begin{proof}
Let $m_t(P)$ and $\Phi_T$ be as in the proof of Theorem~\ref{thm:iProd},
but for the new surrogate loss \eqref{eqn:SquintSurrogateLoss}. Then
$\Phi_T \leq \Phi_{T-1} \leq \ldots \leq \Phi_0 = 1$, because
\[
  \Phi_T - \Phi_{T-1}
    = e^{-\sum_{t=1}^{T-1} m_t(P_t)}
    \left(\E_{P_T}\big[e^{-f_t(\eta,i)}\big]-1\right)
    \leq e^{-\sum_{t=1}^{T-1} m_t(P_t)} \E_{P_T}\big[\eta r_T(i)\big] 
    = 0,
\]
where the inequality follows from the `prod bound' (see the proof of
Theorem~\ref{thm:iProd}) and the final equality is again by definition
of~$\w_T$. For any $Q$, it follows that
\[
  0 \leq \sum_{t=1}^T m_t(P_t) = \sum_{t=1}^T \E_{Q}[\sloss_t(\eta,i)]
  + \sregret(Q) = \sum_{t=1}^T \E_{Q}[-\eta r_t(i) + \eta^2 r_t(i)^2]
  + \sregret(Q),
\]
which implies that \eqref{eqn:iProdReduction} also holds for Squint.
Since \eqref{eqn:iProdSpecial} is a corollary, it also follows directly.
\end{proof}

\subsection{Proof of Theorem~\ref{thm:coinbetting}}
\label{sec:coinbettingproof}

The proof of Theorem~\ref{thm:coinbetting} follows the same general
steps as the proofs for Theorems~\ref{thm:iProd} and~\ref{thm:Squint}.
However, bounding the mix-regret $\sregret_T^i(\eta)$ using a similar
analysis as for the Krichevsky-Trofimov estimator from
Example~\ref{ex:KT} would lead to an extra $\ln T$ factor in the
regret. This is avoided using a more delicate analysis that holds
specifically for the regret with respect to $\eta = \regret_T^+(i)/T$,
which requires a technical analytic inequality by
\citet[Lemma~16]{OrabonaPal2016}.

\begin{proof}
For $\sloss_t$ as in \eqref{eqn:iProdSurrogateLoss}, let $m_t = -\ln
\E_{i \sim P_t}\big[e^{-\sloss_t(\eta_t^i,i)}\big]$. Then, by the
same argument as in the proof of Theorem~\ref{thm:iProd}, $\Phi_T =
e^{-\sum_{t=1}^T m_t} = 1$. For any distribution $\postpi$ on $i$
and any $\etaopt^i \in [0,1]$, we therefore have
\begin{align}
  0 &= \sum_{t=1}^T m_t
    = \E_{\postpi}\left[\sum_{t=1}^T \sloss_t(\eta_t^i,i)\right] +
      \altsregret_T(\postpi)
    \leq \E_{\postpi}\left[\sum_{t=1}^T \sloss_t^i(\eta_t^i)\right] +
      \altsregret_T(\postpi)\notag\\
    &= \E_{\postpi}\left[\sum_{t=1}^T \sloss_t^i(\etaopt^i)
      + \sregret_T^i(\etaopt^i)\right] +
      \altsregret_T(\postpi).\label{eqn:coinbettingmixloss}
\end{align}
The minimizer of $\sum_{t=1}^T \sloss_t^i(\eta)$ over $\eta \in [0,1]$
is $\etaopt^i = \regret_T^+(i)/T$. Plugging this in, we find that
\begin{equation}\label{eqn:coinbettingcomparator}
  \sum_{t=1}^T \sloss_t^i(\etaopt^i)
    = -T \bernoullikl(\thalf + \tfrac{\regret_T^+(i)}{2T} \| \thalf).
\end{equation}
Substituting \eqref{eqn:coinbettingcomparator} in
\eqref{eqn:coinbettingmixloss} and reorganizing we obtain
\eqref{eqn:coinbettinggeneral}.

If we specialize to EW, then $\altsregret_T(\postpi) \leq
\kl(\postpi\|\pi)$ by the same argument as for iProd. In addition, to
bound $\sregret_T^i(\etaopt^i)$, let $\altbeta(x,y)$ be the distribution
on $\eta \in [-1,+1]$ such that $(1+\eta)/2$ has a $\beta(x,y)$
distribution. Then Lemma~\ref{lem:EWRegret} and the observation that the
mixability gap is at most $0$ because $\ell_t^i$ is $1$-exp-concave,
together imply that 
\begin{align*}
  \sregret_T^i(\etaopt^i)
    &\leq \min_{Q \in \domainP} \Big\{\underbrace{\E_{\eta \sim
    Q}\big[\sum_{t=1}^T \sloss_t^i(\eta)\big] + \kl(Q \|
    \altbeta(a,a))}_{A(Q,i)}\Big\} - \underbrace{\sum_{t=1}^T
    \sloss_t^i(\etaopt^i)}_{B(i)}.
\end{align*}
We first rewrite $B(i)$ using \eqref{eqn:coinbettingcomparator}. Then it
remains to bound the term with $A(Q,i)$ in expectation under $\postpi$.
To this end we may assume that $\regret_T(\postpi) :=
\E_{\postpi}[\regret_T(i)] \geq 0$ without loss of generality (otherwise
\eqref{eqn:coinbettingEW} holds trivially). Hence 
\begin{align*}
  \E_{i \sim \postpi}&\big[\min_{Q \in \domainP} A(Q,i)\big]
    \leq \min_{Q \in \domainP} \E_{i \sim \postpi}\big[A(Q,i)\big]\\
    &= \min_{Q \in \domainP} \left\{\E_{\eta \sim Q}\Big[
    -\frac{T+\regret_T(\postpi)}{2} \ln \frac{1+\eta}{2}
      -\frac{T-\regret_T(\postpi)}{2} \ln \frac{1-\eta}{2}
      -T\ln 2\Big] + \kl(Q \| \altbeta(a,a))\right\}\\
    &= -\ln \left(2^T \E_{X \sim \beta(a,a)}\left[X^{\frac{T+\regret_T(\postpi)}{2}}
    (1-X)^{\frac{T-\regret_T(\postpi)}{2}}\right]\right)\\
    &= -\ln \left(\frac{2^T \Gamma(2a)\Gamma\big(\frac{T+\regret_T(\postpi)}{2} +
    a\big)\Gamma\big(\frac{T-\regret_T(\postpi)}{2} + a\big)}{\Gamma(a)^2\Gamma(T+2a)}\right)\\
    &\leq \frac{-\regret_T(\postpi)^2}{2T+4a-2} + \half \ln
    \frac{T+2a-1}{2a}+\ln(e\sqrt{\pi}),
\end{align*}
where we have plugged in the minimizing $Q =
\altbeta(\frac{T+\regret_T(\postpi)}{2} +
a,\frac{T-\regret_T(\postpi)}{2} + a)$, which has nonnegative mean under
our assumption that $\regret_T(\postpi) \geq 0$, and where the last
inequality holds by \citep[Lemma~16]{OrabonaPal2016}, which applies for
$a \geq 1/2$, $\regret_T(\postpi) \in [-T,T]$ and $T \geq 1$.

With these regret bounds for EW, \eqref{eqn:coinbettinggeneral}
specializes to
\[
  \regret_T(\postpi)
    \leq \sqrt{(2T+4a-2)\left(\half \ln \frac{T+2a-1}{2a}+\ln(e\sqrt{\pi}) +
    \kl(\postpi\|\pi)\right)}.
\]
The result so far holds for any $a \geq \half$. Plugging in the choice
$a = \frac{T}{4} + \half$, suggested by \citet{OrabonaPal2016}, and
using $\half \ln \frac{3T}{T+2}+\ln(e\sqrt{\pi}) \leq 3$ completes
the proof.
\end{proof}

\section{Analysis of the Algorithm from Section~\ref{sec:bandits}}
\label{sec:appendix_bandits}
Let $\domainw \subset \reals^d$ be a compact convex set. Following
\citet{Bubeck12}, we assume without loss of generality that $\domainw$
is full rank, meaning that the linear combinations of $\domainw$ span
$\reals^d$ (otherwise we can express the elements of $\domainw$ in a
lower dimensional space). 

At trials
$t=1,2,\ldots,T$, the algorithm plays with a randomized choice $\w_t \in \domainw$, the
adversary chooses
an unobserved loss vector $\g_t$, which is not allowed to depend on the
realization of $\w_t$,
and the learner suffers and observes 
bounded loss $\inner{\w_t}{\g_t}$.
The goal is to minimize the expected regret:
$\E[\regret_T(\u)] = \E\big[\sum_{t=1}^T \inner{\w_t-\u}{\g_t} \big]$ for any
choice of the comparator $\u \in \domainw$.
We consider EW with a fixed learning rate $\eta$
and a prior distribution $P_1$ that is uniform over~$\domainw$.
At each trial $t$, after observing the loss $\inner{\w_t}{\g_t}$, the algorithm
constructs a random, unbiased estimate $\tilde{\g}_t$ of the loss vector $\g_t$
(described below), and uses this estimate to update the posterior. 
Since the projection step can be dropped (as $P_1$ is supported
on $\domainw$), the greedy and lazy versions of EW coincide and the posterior is given by 
$\der P_t(\w) \propto \exp(-\eta \sum_{s = 1}^{t-1} \inner{\w}{\tilde{\g}_s})\der \w$ for
all $\w \in \domainw$. Defining $\np_t = -\eta \sum_{s=1}^{t-1} \tilde{\g}_s$
(with $\np_1 = \0$), we can concisely write: 
\[
  \der P_{t+1}(\w) = e^{\inner{\w}{\np_t} - F(\np_t)}\der \w \quad \forall \w \in \domainw,
  \qquad \text{where~~}
  F(\np) = \ln \int_{\domainw} e^{\inner{\w}{\np}} \; \mathrm{d} \w
\]
is the cumulant generating function. 
At trial $t$,
the EW algorithm samples $\w_t \sim Q_t$, where $Q_t = (1-\gamma) P_t + \gamma R$
for $\gamma \in (0,1)$ is a mixture of the posterior 
$P_t$ and a fixed ``exploration'' distribution~$R$.
The exploration distribution is chosen to be \emph{John's exploration}, defined
as follows \citep{Bubeck12}.
Let $\mathcal{K}$ be the
ellipsoid of minimal volume \emph{enclosing} $\domainw$:
\begin{equation}
  \mathcal{K} = \{\w \in \reals^d \colon (\w-\w_0)^\top \boldsymbol{H}^{-1} (\w-\w_0)
  \leq 1\}
  \label{eq:ellipsoid_E_of_minimal_volume}
\end{equation}
for some positive definite matrix $\boldsymbol{H}$ and $\w_0 \in \reals^d$. In what follows we assume without loss
of generality that $\domainw$ is
centered in the sense that $\w_0=\0$ (otherwise all $\w \in \domainw$ need to be 
shifted by $\w_0$). \citet{Bubeck12} show
that one can choose $M \leq d(d+1)/2 + 1$ contact points $\u_1,\ldots,\u_M \in 
\mathcal{K} \cap 
\domainw$, and a distribution
$R$ over these points that satisfies:
\begin{equation}
  \E_{\w \sim R} [\w\w^\top] = \frac{1}{d}
  \boldsymbol{H}.
\label{eq:property_of_R}
\end{equation}
The estimate $\tilde{\g}_t$ is constructed based on the
observed loss $\inner{\w_t}{\x_t}$, by: 
\[
  \tilde{\g}_t = \inner{\w_t}{\g_t} \left(\E_{Q_t}[\w\w^\top]\right)^{-1} \w_t.
\]
We now show the following regret bound for the resulting algorithm:
\begin{theorem}
Assume the losses are bounded: $|\inner{\w}{\g_t}| \leq 1$ for all $\w \in \domainw$
and all $t$. Let $\eta = \sqrt{\frac{\nu \ln T}{3 d T}}$,
where $\nu = O(d)$ is the self-concordant barrier parameter of $F^*$,
and let $\gamma = \eta d$. Then the expected regret for the EW algorithm
described above is bounded by
\[
  \E[\regret_T(\u)] \leq 2 \sqrt{3\nu d T \ln T} +2
= O(d \sqrt{T \ln T}).
\]
\end{theorem}
\begin{proof}
  We first verify that the estimate $\tilde{\g}_t$ of $\g_t$ is unbiased:
\[
\E_{\w_t \sim Q_t}[\tilde{\g}_t] 
= \E_{\w_t \sim Q_t}\left[\left(\E_{\w \sim Q_t}[\w\w^\top]\right)^{-1} \w_t \inner{\w_t}{\g_t}  \right]
= \left(\E_{\w \sim Q_t}[\w\w^\top]\right)^{-1} \E_{\w_t \sim Q_t}\left[\w_t \w_t^\top \right] \g_t
= \g_t.
\]
Furthermore, due to the inclusion of the exploration
distribution $R$, we have:
\[
  \E_{\w \sim Q_t}[\w\w^\top]
  = (1-\gamma) \E_{\w \sim P_t}[\w\w^\top] + \gamma \E_{\w \sim R}[\w\w^\top]
  \succeq \frac{\gamma}{d} \boldsymbol{H},
\]
(where $\boldsymbol{A} \succeq \boldsymbol{B}$ means $\boldsymbol{A} - \boldsymbol{B}$
is positive semidefinite), and hence for any $\u \in \domainw$:
\begin{equation}
\left \langle \u, \Big(\E_{\w \sim Q_t}[\w\w^\top]\Big)^{-1} \u \right \rangle
\leq \left \langle \u, \frac{d}{\gamma} \boldsymbol{H}^{-1} \u \right \rangle
\leq \frac{d}{\gamma}, 
  \label{eq:quadratic_form_bound}
\end{equation}
where the last inequality is from the fact that $\domainw \subseteq \mathcal{K}$
and from the definition of $\mathcal{K}$ in \eqref{eq:ellipsoid_E_of_minimal_volume}.
This, however, implies that the linear losses induced by $\tilde{\g}_t$
are bounded for any $\u \in \domainw$:
\begin{align}
  \inner{\u}{\tilde{\g}_t} 
  &= \inner{\w_t}{\g_t} \left \langle \u,
    \Big(\E_{\w \sim Q_t}[\w\w^\top]\Big)^{-1} \w_t \right \rangle \nonumber \\
  &\leq |\inner{\w_t}{\g_t}| \left \langle \w_t,
  \Big(\E_{\w \sim Q_t}[\w\w^\top]\Big)^{-1} \w_t \right \rangle^{1/2}
    \left \langle \u,
  \Big(\E_{\w \sim Q_t}[\w\w^\top]\Big)^{-1} \u \right \rangle^{1/2} \leq \frac{d}{\gamma},
  \label{eq:bound_on_the_loss_estimate}
\end{align}
where the first inequality is from the Cauchy-Schwarz inequality
(for positive semidefinite $\boldsymbol{A}$, 
$\boldsymbol{x}^\top \boldsymbol{A} \boldsymbol{y}
\leq (\boldsymbol{x}^\top \boldsymbol{A} \boldsymbol{x})^{1/2}
(\boldsymbol{y}^\top \boldsymbol{A} \boldsymbol{y})^{1/2}$),
while the second inequality is due to 
assumption $|\inner{\w}{\g_t}| \leq 1$ and 
due to \eqref{eq:quadratic_form_bound} applied twice (first to $\u$ and
then to $\w_t$).

Let $\Mu_t$ be the mean value of $P_t$:
$\Mu_t = \E_{P_t}[\w]$. As a general property of exponential families or as a consequence of Theorem~\ref{thm:MDasEW}, we have
$\Mu_t = \nabla F(\np_t)$,
and $\Mu_t$ and $\np_t$ are conjugate parameters of the exponential family.
Let us fix a comparator $\u \in \domainw$ and define $P_\u$ to be the
member of the exponential
family with cumulant generating function $F$ that has mean value $\u$:
$\E_{\w \sim P_\u}[\w]=\u$. 
We now apply Lemma~\ref{lem:EWRegret} for the EW algorithm on the
sequence of linear losses induced by
$\tilde{\g}_1,\ldots,\tilde{\g}_T$ to get:
\begin{align*}
\sum_{t=1}^T \inner{\Mu_t - \u}{\tilde{\g}_t}
&= \sum_{t=1}^T \inner{\Mu_t}{\tilde{\g}_t}
- \sum_{t=1}^T \E_{\w \sim P_\u}[\inner{\w}{\tilde{\g}_t}] \\
&\leq \frac{1}{\eta} \kl(P_\u \| P_1)
+ \sum_{t=1}^T \inner{\Mu_t}{\tilde{\g}_t} + 
\frac{1}{\eta} \ln \E_{\w \sim P_t} \left[e^{-\eta \inner{\w}{\tilde{\g}_t}} \right]
\end{align*}
(note that in this section we use $\Mu_t$ to denote the mean of $P_t$, while
$\w_t$ is reserved for the randomized action at trial $t$ sampled from $Q_t$).
Since $P_\u$ and $P_1$ are members of the same exponential family, 
the KL-term can be re-expressed using Lemma~\ref{th:KLisB}:
\[
  \kl(P_\u \| P_1) = D_{F^*}(\u \| \Mu_1)
  = F^*(\u) - F^*(\Mu_1) - \underbrace{\nabla F^*(\Mu_1)^\top}_{\0}
  (\Mu - \Mu_1)
  = F^*(\u) - F^*(\Mu_1),
\]
where we used the fact that $\Mu_1$ has conjugate parameter $\np_1 = \0$, and
thus $\nabla F^*(\Mu_1) = \np_1 = \0$.
To bound the mixability gap,
we will now use that by assumption $\eta = \frac{\gamma}{d}$, so that by
\eqref{eq:bound_on_the_loss_estimate} we have
$|\eta \inner{\w}{\tilde{\g}_t}| \leq 1$ for any $\w \in \domainw$. 
Using the fact that $e^{-s} \leq 1 - s + s^2$
holds for $s \geq -1$, and 
combining with $\ln(1+x) \leq x$ gives:
\begin{align*}
\inner{\Mu_t}{\tilde{\g}_t} + 
\frac{1}{\eta} \ln \E_{\w \sim P_t} \left[e^{-\eta \inner{\w}{\tilde{\g}_t}} \right]
  &\leq \inner{\Mu_t}{\tilde{\g}_t} + 
\frac{1}{\eta} \ln \left(1 + \E_{\w \sim P_t} \left[-\eta \inner{\w}{\tilde{\g}_t} +
  \eta^2 \inner{\w}{\tilde{\g}_t}^2 \right] \right) \\
  &\leq \underbrace{\inner{\Mu_t}{\tilde{\g}_t} - \E_{\w \sim P_t}[\inner{\w}{\tilde{\g}_t}]}_{=0}
  + \eta \E_{\w \sim P_t} \left[\inner{\w}{\tilde{\g}_t}^2 \right] \\
  &= \eta \tilde{\g}_t^\top \E_{\w \sim P_t} \left[\w \w^\top \right] \tilde{\g}_t.
\end{align*}
Combining the bounds on the KL-term and the mixability gap gives:
\begin{equation}
\sum_{t=1}^T \inner{\Mu_t - \u}{\tilde{\g}_t}
\leq \frac{F^*(\u) - F^*(\Mu_1)}{\eta} + \eta \sum_{t=1}^T 
 \tilde{\g}_t^\top \E_{\w \sim P_t} \left[\w \w^\top \right] \tilde{\g}_t.
\label{eq:EW_bound_on_mu_t}
\end{equation}
We can use this result to bound the regret of the original algorithm in the
following way. 
First, note that:
\begin{align*}
  \E_{\w_t \sim Q_t} \left[ \inner{\w_t-\u}{\g_t} \right] 
&= \gamma \langle \E_{\w_t \sim R}[\w_t]-\u, \g_t\rangle
+ (1-\gamma) \langle \E_{\w_t \sim P_t}[\w_t] - \u, \g_t \big\rangle \\
&\leq 2\gamma + (1-\gamma) \langle \Mu_t - \u, \g_t \big\rangle 
=  2\gamma + (1-\gamma) \E_{\w_t \sim Q_t} \left[ \inner{\Mu_t-\u}{\tilde{\g}_t} \right],
\end{align*}
where the random quantity in the last expectation is $\tilde{\g}_t$,
because it depends on $\w_t$.
Therefore:
\begin{align}
\sum_{t=1}^T \E_{\w_t \sim Q_t} [\inner{\w_t - \u}{\g_t}]
&\leq 2 \gamma T + (1-\gamma) \sum_{t=1}^T \E_{\w_t \sim Q_t} 
\left[ \inner{\Mu_t-\u}{\tilde{\g}_t} \right] \nonumber \\
&\leq 2\gamma T + \frac{F^*(\u) - F^*(\Mu_1)}{\eta} +  \eta (1-\gamma)
\sum_{t=1}^T \E_{\w_t \sim Q_t} \left[\tilde{\g}_t^\top \E_{\w \sim P_t} \left[\w
\w^\top \right] \tilde{\g}_t \right] \nonumber \\
  &\leq 2\gamma T +  \frac{F^*(\u) - F^*(\Mu_1)}{\eta} + \eta
\sum_{t=1}^T \E_{\w_t \sim Q_t} \left[\tilde{\g}_t^\top \E_{\w \sim Q_t} \left[\w
\w^\top \right] \tilde{\g}_t \right],
\label{eq:intermediate_bound_on_quadratic_form}
\end{align}
where the second inequality is from \eqref{eq:EW_bound_on_mu_t},
while the last inequality is due to:
\[
  \E_{\w \sim Q_t}[\w\w^\top]
  = (1-\gamma) \E_{\w \sim P_t}[\w\w^\top] + \gamma \E_{\w \sim R}[\w\w^\top]
  \succeq (1-\gamma) \E_{\w \sim P_t}[\w\w^\top].
\]
Using the definition of $\tilde{\g}_t$ and $\inner{\w_t}{\g_t}^2 \leq
1$, we further bound:
\begin{align*}
  \E_{\w_t \sim Q_t} \left[ \tilde{\g}_t^\top \E_{\w \sim Q_t} \left[\w
\w^\top \right] \tilde{\g}_t \right] 
&\leq  \E_{\w_t \sim Q_t} \left[\w_t^\top \left(\E_{\w \sim Q_t}[\w\w^\top]\right)^{-1} \E_{\w \sim Q_t} \left[\w \w^\top \right] \left(\E_{\w \sim Q_t}[\w\w^\top]\right)^{-1} \w_t \right] \\
&= \sum_{t=1}^T \E_{\w_t \sim Q_t} \left[\Tr\left(
  \left(\E_{\w \sim Q_t}[\w\w^\top]\right)^{-1} \w_t \w_t^\top \right) \right] \\
&= \sum_{t=1}^T \Tr\left( \I \right) = Td.
\end{align*}
Plugging the above into \eqref{eq:intermediate_bound_on_quadratic_form} and taking
expectation with respect to the randomness of the algorithm results in
the following bound on
the expected regret:
\[
\E [\regret_T(\u)] 
= \E \left[\sum_{t=1}^T \E_{\w_t \sim Q_t} [\inner{\w_t - \u}{\g_t}]
 \right] 
 \leq 2\gamma T +  \frac{F^*(\u) - F^*(\Mu_1)}{\eta} + \eta Td.
\]
What is left to bound is $F^*(\u) - F^*(\Mu_1)$. 
To this end, define
the Minkowski function \citep{Scrible}
on $\domainw$ as:
\[
\pi_{\Mu}(\w) = \inf\{t \geq 0 \colon \Mu + t^{-1}(\w - \Mu) \in \domainw\}.
\]
\citet{bubeck2014} show
that $F^*$
is a $\nu$-self concordant barrier on $\mathcal{\domainw}$ with $\nu = O(d)$. 
Using this property and Theorem 2.2 from \cite{Scrible} we get:
\[
  F^*(\u) - F^*(\Mu_1) \leq \nu \ln \left(\frac{1}{1-\pi_{\Mu_1}(\u)}\right).
\]
If $\u$ is such that $\pi_{\Mu_1}(\u) \leq 1 - \frac{1}{T}$, then
$F^*(\u) - F^*(\Mu_1) \leq \nu \ln T$.
On the other hand, if $\pi_{\Mu_1}(\u) \leq 1 - \frac{1}{T}$, we define
a new comparator $\u' = (1-\frac{1}{T}) \u + \frac{1}{T}\Mu_1$, for which
$\pi_{\Mu_1}(\u') \leq 1 - \frac{1}{T}$ \citep{Scrible}, and use the regret bound above
for $\u'$ to get:
\begin{align*}
  \E [\regret_T(\u)] &=\E [\regret_T(\u')] + \sum_{t=1}^T \inner{\u' - \u}{\g_t} 
                     =\E [\regret_T(\u')] + \frac{1}{T} \sum_{t=1}^T \inner{\Mu_1-\u}{\g_t}  \\
                     &\leq 2\gamma T +  \frac{F^*(\u') - F^*(\Mu_1)}{\eta} + \eta T d + 2
                     \leq  2\gamma T + \frac{\nu \ln T}{\eta} +  \eta T
                     d + 2.
\end{align*}
Recalling that $\gamma = \eta d$ and tuning $\eta = \sqrt{\frac{\nu \ln T}{3 d T}}$
gives the claimed bound.
\end{proof}
\end{document}